%% file: main.tex
\documentclass[10pt,twocolumn,letterpaper]{article}

\usepackage{cvpr}              % To produce the CAMERA-READY version
\usepackage{graphicx}
\usepackage{amsmath}
\usepackage{amssymb}
\usepackage{booktabs}
\usepackage{multirow}
\usepackage{multicol}

\usepackage{color}
\usepackage{marvosym}
\input{math_commands.tex}

\usepackage[accsupp]{axessibility}

\usepackage[pagebackref,breaklinks,colorlinks]{hyperref}
\usepackage{amsfonts,bm}
\usepackage{mathtools}
\usepackage{amsthm}

\usepackage{tabularray}

\usepackage[ruled,linesnumbered]{algorithm2e}
\usepackage{algorithmic}
\usepackage{chemformula}
\usepackage[capitalize]{cleveref}

\crefname{section}{Sec.}{Secs.}
\Crefname{section}{Section}{Sections}
\Crefname{table}{Table}{Tables}
\crefname{table}{Tab.}{Tabs.}

\theoremstyle{plain}
\newtheorem{theorem}{Theorem}

\newtheorem{lemma}{Lemma}
\newtheorem{corollary}{Corollary}
\theoremstyle{definition}

\newtheorem{assumption}{Assumption}
\newtheorem{remark}{Remark}

\def\paperID{13579} % *** Enter the Paper ID here
\def\confName{CVPR}
\def\confYear{2024}

\title{Decentralized Directed Collaboration for Personalized Federated Learning }

\author{
Yingqi Liu\textsuperscript{\rm 1}
\quad
Yifan Shi\textsuperscript{\rm 2} 
\quad
Qinglun Li \textsuperscript{\rm 3} 
\quad
Baoyuan Wu\textsuperscript{\rm 4} 
\quad
Xueqian Wang\textsuperscript{\rm 2}
\quad %\textsuperscript{\rm 5 }
Li Shen\textsuperscript{\rm 5 }\thanks{Corresponding Author.}\\ %\thanks{Corresponding author: Li Shen}\\
\textsuperscript{\rm 1}Nanjing University of Science and Technology, Nanjing, China;
\textsuperscript{\rm 2}Tsinghua University, Shenzhen, China; \\
\textsuperscript{\rm 3}National University of Defense Technology, Changsha, China;\\
\textsuperscript{\rm 4}The Chinese University of Hong Kong, Shenzhen, China; 
\textsuperscript{\rm 5}JD Explore Academy; Beijing, China.\\
{\tt\small lyq@njust.edu.cn;
shiyf21@mails.tsinghua.edu.cn;
liqinglun@nudt.edu.cn;}\\
{\tt\small wubaoyuan@cuhk.edu.cn;
wang.xq@sz.tsinghua.edu.cn; 
mathshenli@gmail.com.}
}
\begin{document}
\maketitle
\input{content/01_abstract.tex}
\input{content/02_Introduction.tex}

\input{content/03_related_work.tex}
\input{content/04_prelimary.tex}

\input{content/05_algorithm.tex}

\input{content/06_theory.tex}

\input{content/07_experiment}
\input{content/08_conclusion.tex}

\clearpage
{
 \small
\bibliographystyle{ieeenat_fullname} %{}
\bibliography{ref}
}

\clearpage
\input{content/09_appendix.tex}

% WARNING: do not forget to delete the supplementary pages from your submission 
% \input{sec/X_suppl}

\end{document}

%% file: math_commands.tex
\usepackage{amsmath,amsfonts,bm}

\def\eqref#1{equation~\ref{#1}}
\def\1{\bm{1}}

\DeclareMathAlphabet{\mathsfit}{\encodingdefault}{\sfdefault}{m}{sl}
\SetMathAlphabet{\mathsfit}{bold}{\encodingdefault}{\sfdefault}{bx}{n}

\newcommand{\E}{\mathbb{E}}

%% file: content/01_abstract.tex
\begin{abstract}

Personalized Federated Learning (PFL) is proposed to find the greatest personalized models for each client. To avoid the central failure and communication bottleneck in the server-based FL, we concentrate on the Decentralized Personalized Federated Learning (DPFL) that performs distributed model training in a Peer-to-Peer (P2P) manner. Most personalized works in DPFL are based on undirected and symmetric topologies, however, the data, computation and communication resources heterogeneity result in large variances in the personalized models, which lead the undirected aggregation to suboptimal personalized performance and unguaranteed convergence. 
To address these issues, we propose a directed collaboration DPFL framework by incorporating stochastic gradient push and partial model personalized, called \textbf{D}ecentralized \textbf{Fed}erated \textbf{P}artial \textbf{G}radient \textbf{P}ush (\textbf{DFedPGP}). It personalizes the linear classifier in the modern deep model to customize the local solution and learns a consensus representation in a fully decentralized manner. Clients only share gradients with a subset of neighbors based on the directed and asymmetric topologies, which guarantees flexible choices for resource efficiency and better convergence.
Theoretically, we show that the proposed DFedPGP achieves a superior convergence rate of 
$\mathcal{O}(\frac{1}{\sqrt{T}})$ 
%$\mathcal{O}(1/\sqrt{T})$ 
in the general non-convex setting, and prove the tighter connectivity among clients will speed up the convergence. 
The proposed method achieves state-of-the-art (SOTA) accuracy in both data and computation heterogeneity scenarios, demonstrating the efficiency of the directed collaboration and partial gradient push.

\end{abstract}

%% file: content/02_introduction.tex
\vspace{-0.4cm}
\section{Introduction}

\begin{figure}[th] 
\centering
\includegraphics[width=0.4\textwidth]{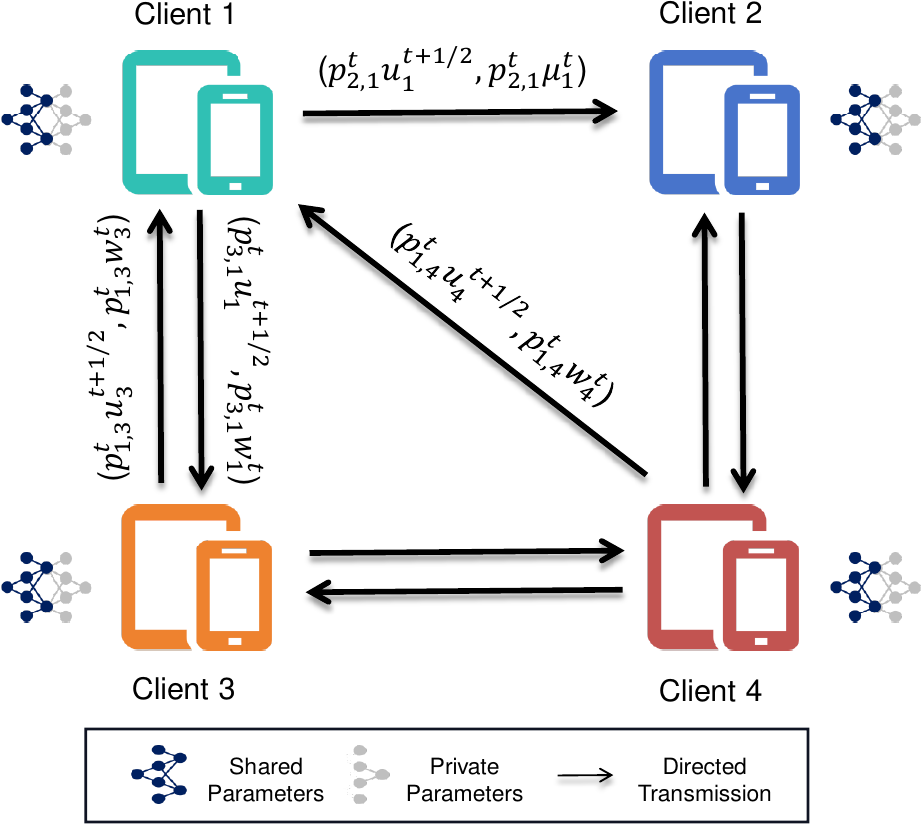}
\centering
% \vspace{-0.35cm}
\caption{\small An overview of the DFedPGP with a directed graph. We take Client 1 as an example. It pushes the shared parameters $p^t_{j,1},u_1^{t+1/2}$ and bias information $p_{j,1}^t,\mu_1^t$ to its out-neighbors (Client 2, 3); pulls the shared parameters $p_{1,j}^t,u_j^{t+1/2}$ and bias information $p_{1,j}^t,\mu_j^t$ from its in-neighbors (Client 3, 4).
}\label{fig:dfedpgp}
 \vspace{-0.3cm}
\label{fig:abla}
\end{figure}

Recently, Personalized Federated Learning (PFL) has emerged to find the best model for each client since one consensus model can not satisfy all clients' needs in classical Federated Learning (FL) \citep{ye2023heterogeneous}. The existing PFL algorithm can be categorized into two branches in terms of the existence of the centralized server (i.e., Centralized Personalized Federated Learning (CPFL) \citep{arivazhagan2019federated,lin2020ensemble,huang2021personalized,oh2021fedbabu} and Decentralized Personalized Federated Learning (DPFL) \citep{Rong2022DisPFL,li2023effectiveness,sui2022find}). The challenges of centralized communication bottleneck or central failure may incur low communication efficiency or system crash in the federated processing. Thus, we focus on the DPFL, which allows edge clients to communicate with each other in a peer-to-peer manner, aiming to reduce the communication column of the busiest server node and embrace peer-to-peer communication for faster convergence. In decentralized FL, clients usually follow an undirected and symmetric communication topology to reach a consensus model \citep{shi2023improving,sun2022decentralized, Rong2022DisPFL}, which means if one client receives neighbors' models, it sends its model back.

In order to satisfy the unique needs of individual clients, most existing works in PFL carefully designed the relationships between the global model and personalized models to fit the local data distribution via different techniques, such as parameter decoupling \citep{arivazhagan2019federated, collins2021exploiting, oh2021fedbabu}, knowledge distillation  \citep{li2019fedmd, lin2020ensemble, he2020group}, multi-task learning  \citep{huang2021personalized, shoham2019overcoming}, model interpolation  \citep{deng2020adaptive, diao2020heterofl} and clustering  \citep{ghosh2020efficient, sattler2020clustered}. These techniques can also be adopted to improve the personalized performance in DPFL \citep{Rong2022DisPFL,li2022learning}. However, the heterogeneity among clients exists not only in local data distribution but also in the communication power and computation resources \citep{chen2023fs,chen2023efficient,chai2019towards}. The power level of the wireless channel among clients may be different and time-varying in communication networks, and some clients may get offline occasionally without sending messages to their neighbors. These result in long-term waits or incidents of deadlock for their neighbors \citep{chen2023enhancing,yu2019linear} and also lead to poor convergence for the whole system. Besides, there is no reason to expect that the exchanged models are trained at the same convergence level due to the heterogeneous computation resources. Clients may receive excessive poor-performing models which can not help their training and degrade the personalized performance.

To tackle the challenges above, we propose a DPFL framework with a directed communication topology, termed DFedPGP, which incorporates the partial model personalization and stochastic gradient push to boost the personalized performance of the heterogeneous clients. Both partial model personalization and stochastic gradient push contribute to speeding up the convergence and reducing the communication resources to reach an ideal performance. Instead of exchanging the full model with their undirected neighbors, we decouple the model as a mixture of a shared feature representation part and a private linear classifier part and only push the shared partial gradients to the directed out-neighbors (as depicted in Figure \ref{fig:dfedpgp} ). Specifically, the proposed method consists of three steps: (1) pull the shared partial gradient and the bias weights from in-neighbors; (2) local update the personalized linear classifier and the shared feature representation alternately with the de-biased parameters; (3) push the updated shared gradients and the bias information to out-neighbors. In-neighbors and out-neighbors are the in-coming and out-coming links for each client here. Partial gradient push makes the personalized information well stored in the private linear classifier, reducing communication costs as well as protecting clients’ privacy. Moreover, directed contact allows clients to choose their neighbors flexibly, meaning that the shared part model has a larger feature search space among clients, which guarantees better performance in a computation-constrained and communication-constrained scenario.

Theoretically, we present the non-trivially convergency analysis for the DFedPGP algorithm (see Section \ref{theory}), which achieves a convergence rate of $\mathcal{O}(\frac{1}{\sqrt{T}})$  in the general non-convex setting. Empirically, we conduct extensive experiments on the CIFAR-10, CIFAR-100, and Tiny-ImageNet datasets in non-IID settings with different data partitions. Experimental results confirm that the proposed algorithm can achieve competitive performance relative to other SOTA baselines (see Section \ref{exper}) in PFL.% In summary, we provide a comprehensive study focusing on directed personalization collaboration in a peer-to-peer manner. 

In summary,  our main contributions are four-fold:
\begin{itemize}
    \item We introduce the directed Push-sum optimization to PFL, which allows clients to choose their neighbors flexibly and guarantees a larger feature search space in a communication, and computation heterogeneity scenario. 
    
    \item We propose a DPFL framework DFedPGP, incorporated with stochastic gradient push and partial model personalization for robust communication and fast convergence. 
    
    \item  We provide convergence guarantees for DFedPGP in the general non-convex setting with peer-to-peer partial participation in DPFL.
    
    \item Empirical results indicate the superiority of the proposed DFedPGP compared with various SOTA baselines and it can be well adapted to the data heterogeneous and computation resources constrained settings.
\end{itemize}

%% file: content/03_related_work.tex
\section{Related Work}

\textbf{Personalized Federated Learning (PFL).}
The PFL aims to produce the greatest personalized models for each client by model decoupling  \citep{arivazhagan2019federated, collins2021exploiting}, knowledge distillation  \citep{li2019fedmd, lin2020ensemble}, multi-task learning  \citep{huang2021personalized, shoham2019overcoming}, model interpolation  \citep{deng2020adaptive, diao2020heterofl} and clustering  \citep{ghosh2020efficient, sattler2020clustered}. More details can be referred to in  \citep{tan2022towards}. In this paper, we mainly focus on the model decoupling methods, which divide the model into a global shared part and a personalized part, also called \emph{partial personalization}. Existing partial personalized works in CFL achieve better performance than full model personalization with fewer shared parameters. FedPer \citep{arivazhagan2019federated}, FedRep \citep{collins2021exploiting} and FedBABU \citep{oh2021fedbabu} all use one global feature representation with many local classifiers but with differences in the relationship between the shared representation and the private linear parts. %LG-FedAvg \citep{liang2020think} keeps the feature representations locally and shares the classifiers. 
Fed-RoD \citep{chen2021bridging} simultaneously trains a global full model and many private classifiers with both class-balanced loss and empirical loss. Theoretically,
FedSim and FedAlt \citep{pillutla2022federated} provide the convergence analyses of both algorithms in the general non-convex setting, while FedAvg-P and Scaffold-P \citep{chen2023sharper} improve the existing results in \citep{pillutla2022federated}. %Inspired by these, we provide the non-trivial convergence analysis on decentralized partial model personalization and deliver theoretical analysis at first combining partial personalization with various peer-to-peer communication networks and the SAM optimizer. 
% For a fair comparison and exploration, we mainly present the related works about partial personalization and full personalization. The latter represents the whole model as the personalized model without decoupling parameters. 

\medskip
\noindent
\textbf{Decentralized Federated Learning (DFL).}
 Due to the computation and communication resources heterogeneity among clients, DFL has been an encouraging field in recent years \citep{beltran2022decentralized,kang2022blockchain, LiJunBlockchain,shi2023improving}, where clients only connect with their neighbors through peer-to-peer communication.  We discuss the PFL methods in DFL considering multi-step local iterations.%\footnote{In decentralized training, they also focus on peer-to-peer communication, but one-step local iteration is adopted, due to the gradient computation being more focused than the communication burden. More detailed related works in decentralized training are placed in \textbf{Appendix} \ref{ap:related_works}.} 
 Specifically, DFedAvgM \citep{sun2022decentralized} applies multiple local iterations with SGD and the quantization method to reduce the communication cost. Dis-PFL \citep{Rong2022DisPFL} customizes the personalized model and pruned mask for each client to speed up the personalized convergence. KD-PDFL \citep{jeong2023personalized} leverages the knowledge distillation technique to empower each device to discern statistical distances between local models. ARDM \citep{sadiev2022decentralized} presents lower bounds on the communication and local computation costs for this personalized FL formulation in a peer-to-peer manner. %To reduce the central server's communication burden and the risk of disruption if the central server fails, in this work, we leverage a decentralized communication way to aggregate the shared model based on model decoupling.

\medskip
\noindent
\textbf{Push-sum over Directed Graphs.} 
Push-sum optimizer is proposed to solve the asymmetric optimization problems over (time-varying) directed graphs. The first Push-sum study in \citep{ kempe2003gossip} discusses gossip-type problems in directed graphs. PS-DDA \citep{tsianos2012push} extends this method to a decentralized scenario and proves the convergence in a convex set. More optimization analysis can be referred to in \citep{nedic2014distributed, xi2015linear,xi2017dextra,xin2018linear,nedic2017achieving}. As an effective optimizer, Push-sum and its variants have been applied to various machine learning (ML) tasks \citep{assran2019stochastic,taheri2020quantized,chen2023enhancing,assran2020asynchronous,li2023asymmetrically}. For example, SGP \citep{assran2019stochastic} combines Push-sum with stochastic gradient updates and also proposes the Overlap SGP, allowing overlaps of communication and computation to hide communication overhead. Quantized Push-sum \citep{taheri2020quantized} quantizes the Push-sum based algorithm over directed graphs to tackle the heavy communication load. AsyNG \citep{chen2023enhancing} proposes an asynchronous DFL system with directed communication by incorporating neighbor selection and gradient push to boost the performance on non-IID local data and heterogeneous edge nodes.

\medskip
Nowadays, almost all PFL works suffer from the risk of deadlock from unstable communication channels and suboptimal convergence from the different convergence-level aggregations. Therefore, we try to propose a framework of partial gradient push based on a directed communication graph for DPFL. It differs from the existing directed DFL methods in the exchange model part like OSGP\citep{assran2019stochastic}, where clients focus on the whole parameters exchange for the only consensus model. Also, we adopt multi-step local steps and multiple alternate optimizations for better convergence, which leads to an unbiased gradient estimation and the dependent stochastic variance between the shared parts and the personal parts. Therefore, the algorithm design and the theoretical analysis are both unique and non-trivial. 

%% file: content/04_prelimary.tex
\section{Methodology}\label{algo}

In this section, we first define decentralized partial personalized models and the directed graph network in DPFL. Then we present the DFedPGP, which leverages the partial gradient push in the directed graph to mitigate the negative impact of heterogeneous data and computation resources. 

\subsection{Problem Setup} 

\textbf{Decenntralized Personalized Federated Learning.} %We first introduce the personalized problem for DFL in non-convex setting and then expand to a partial personalized problem based on it. 
Consider a typical setting of DFL with~$m$ clients, where each client~$i$ has the data distribution $\mathcal{D}_i$. We focus on the minimization of the finite sum of non-convex functions:
\begin{equation}\label{dec}
\begin{split}
    &\small \min_{w\in \mathbb{R}^d} F(w):  =\frac{1}{m}\sum_{i=1}^m F_i(w_i), \\
    &F_i(w_i) =\mathbb{E}_{\xi\sim \mathcal{D}_i} F_i(w_i;\xi).
\end{split}
\end{equation}
where $F:\mathbb{R}^d \to \mathbb{R}$ is the global object function; $w_i \in\mathbb{R}^d$ represents the parameters of the machine learning model in client $i$;  $F_i$ is the loss function associated with the data sample $\xi$ randomly drawn from the distribution $\mathcal{D}_i$ in client~$i$.

To relieve the communication burden and improve personalized performance, we consider the partial model personalized version in DPFL. Specifically, the model parameters are partitioned into two parts: 
the \emph{shared} parameters $u\in\mathbb{R}^{d_0}$ and the \emph{personal}
parameters $v_i\in\mathbb{R}^{d_i}$ for $i=1,\ldots,m$. 
The full model on client~$i$ is denoted as $w_i=(u_i,v_i)$. To simplify presentation, we denote $V=(v_1,\ldots,v_m)\in\mathbb{R}^{d_1+\ldots+d_m}$, and then our goal is to solve this problem:
\begin{equation}\label{eqn:partial PFL in DFL}
\begin{split}
    & \min_{u, V} \quad F(u, V):=\frac{1}{m} \sum_{i=1}^m F_i\left(u, v_i\right),\\
    & F_i\left(u_i, v_i\right)=  {\E}_{\xi \sim \mathcal{D}_i}\left[F_i\left(u_i, v_i; \xi\right)\right].
\end{split}
\end{equation}
where $u$ denotes the consensus model averaged with  $u_i$, i.e., $u = \frac{1}{m}\sum_{i=1}^m u_i$ and we use ${\nabla}_u$ and ${\nabla}_v$ to represent stochastic gradients with respect to~$u_i$ and $v_i$, respectively.

\medskip
\noindent
\textbf{Directed Graph Network.} In the decentralized network topology, the communication between clients can be modeled as a directed connected graph $\mathcal{G}(t) = (\mathcal{N},\mathcal{V}(t),\mathcal{E}(t) )$, where $\mathcal{N} = \{1, 2, \ldots, m\}$ represents the set of clients, $\mathcal{V}(t) \subseteq  \mathcal{N} \times  \mathcal{N}$ represents the set of communication channels and $(i, j) \in  \mathcal{E}(t)$ represents a directed link from client $i$ to client $j$. Considering the time-varying directed graph, the link $(i, j) \in  \mathcal{E}(t) $ (where $i \neq j $) does not imply the link $(j, i) \in \mathcal{E}(t)$. To further describe the directed communication, we define $N_i^{in} = \{ j |(j, i)\in \mathcal{E}(t), j \in \mathcal{N}\}$ as the in-neighbor set and $N_i^{out} = \{j |(i, j) \in \mathcal{E}(t), j \in \mathcal{N}\}$ as the out-neighbor set, which are the sets with in-coming and out-coming links into node $i$ separately.

Most works in DPFL assume the communication is based on a time-varying undirected graph, which satisfies $N_i^{in} = N_i^{out}$ and the link $(i, j) \in  \mathcal{E}(t) $ (where $i \neq j $) must be equal to the link $(j, i) \in \mathcal{E}(t)$. But in reality, the undirected communication graph requires high attention in the implementation to avoid deadlocks. Directed communication graph networks mitigate this issue by flexibly selecting neighbors within clients and exhibiting higher robustness in terms of network communication quality.

%% file: content/05_algorithm.tex
\subsection{Algorithm}

In this section, DFedPGP (see Algorithm \ref{DFedPGP}) is proposed to solve the problem (\ref{eqn:partial PFL in DFL}) in a fully decentralized manner.

%In the DFedPGP algorithm, we combine the stochastic gradient push and partial model personalized strategies to better fit both the data heterogeneity and computation heterogeneity problem. % and illustrated in Figure 1.

%\subsubsection{De-biased Local Update Based on Partial Gradient Push.}

\medskip
\noindent
\textbf{Partial Model Personalization.} Drawing from previous research on CNNs, layers that serve specific engineering purposes: lower convolution layers (close to the input) are responsible for feature extraction, and the upper linear layers (close to the output) focus on complex pattern recognition \citep{pillutla2022federated}. The feature extraction layers, mapping data from high-dimensional feature space to an easily distinguished low space, are similar between clients but prone to over-fitting. The linear classification layers, which determine the data category from the output of the previous feature extraction layers, are very different from data heterogeneity clients \citep{li2023effectiveness}. Therefore, we set the feature extraction layers as the shared parts and the linear classification layers as the personalized parts as \citep{arivazhagan2019federated,collins2021exploiting,oh2021fedbabu,pillutla2022federated}, and we leverage the alternating update approach for model training in Line 5-12, which aims to increase the compatibility between the personalized and the shared parts.

\medskip
\noindent
\textbf{Push-sum Based DFedPGP.} The Push-sum method \citep{nedic2016stochastic} to solve the decentralized optimization problem performs one local stochastic gradient descent update with one iteration of push-pull transmission at each client. It maintains four variables locally at each round $t$: the biased shared model parts parameters $u_i^t$, the private model parts parameters $v_i^t$, the Push-sum bias weight $\mu_i^t$, and the de-biased shared model parts parameters $z_i^t = u_i^t / \mu_i^t$.  To save the overall communication, we introduce an idea from local SGD to perform a few epochs of local training before weights transmission.  So at each round, every client performs a few local SGD steps in Lines 5–12 followed by one step of push-pull transmission in Lines 14–17. Notably, the local gradient is calculated at the de-biased parameters $z_i^t$ in line 6 and they are then used to be updated in Line 10. The push-pull transmission includes the biased shared model parameters $u_i^t$ and the Push-sum bias weight $\mu_i^t$.

\medskip
\noindent
\textbf{Directed Communication Graph.} 
We set the mixing matrix $P^t$ to describe the communication topology at each round $t$. DFedPGP can be adapted to various communication topologies such as time-varying, asymmetric, and sparse networks. We used the time-varying, asymmetric network here to encounter the limited communication bandwidth. Clients only need to know the outgoing mixing weights at each communication round and can independently choose the mixing weights from the other clients in the network. In this work, we introduce a simple yet effective random client selection method that satisfies our theory (Section \ref{theory}) and the limited communication bandwidth in the experiments (Section \ref{exper}). %The client selection method can be seen in Appendix \ref{ap:selection}.

\begin{algorithm}[t]
\small
\caption{DFedPGP}
\label{DFedPGP}
% \centering
\SetKwData{Left}{left}\SetKwData{This}{this}\SetKwData{Up}{up} \SetKwFunction{Union}{Union}\SetKwFunction{FindCompress}{FindCompress}
\SetKwInOut{Input}{Input}\SetKwInOut{Output}{Output}

\Input{Total number of devices $m$, total number of communication rounds $T$,  local learning rate $\eta_{u}$ and $\eta_{v}$, total number of local iterates $K_{u}$ and $K_{v}$ and mixing weight $p^{t}_{j,i} = {1/ \vert {\mathcal{N}^{out}_{i,0}} \vert} $ .}

\Output{Personalized model $  {u}^T_{i} = z^T_{i}$ and $  {v}^T_{i}$ after the final communication of all clients.}

\textbf{Initialization:}  Randomly initialize each device's shared parameters $  {u}^{0}_{i}$, the de-biased shared parameters $ z^{0}_{i} = {u}^{0}_{i} $ , personal parameters $  {v}^{0}_{i}$ and push-sum weight $\mu^{0}_{i} = 1 $.

\For{$t=0$ \KwTo $T-1$}{
    \For{client $i$ in parallel }{
        Set $  {u}^{t,0}_{i} \gets    {u}^t_{i}$ and sample a batch of local data $\xi_i$ and calculate local gradient iteration.\\
        \For{$k=0$ \KwTo $K_{v}-1$ }{  
        Perform personal parameters $  {v}_i$ update: 
        ${  {v}_{i}^{t, k + 1}} = {  {v}_{i}^{t, k}} - {\eta_v}{\nabla_v}F_i({z}^{ t,0}_{i},{  {v}_{i}^{t,k}};{\xi _i})$.}

        $  {v}_{i}^{t + 1} \gets   {v}_{i}^{t, K_v}$. 
        
        \For{$k=0$ \KwTo $K_{u}-1$ }
            {Update shared parameters $  {u}_i$ via ${  {u}_{i}^{t, k + 1}} =   {u}_{i}^{t, k} - {\eta _u}\nabla_u F_i(z_{i}^{t, k}, {v}_{i}^{t + 1};\xi_i)$. 
            
            ${z}_{i}^{t,k+1} = {u}_{i}^{t,k+1}/ \mu^{t}_{i}$.}

        $  {u}_{i}^{t + 1/2} \gets   {u}_{i}^{t, K_u}$.

        Push weights $p_{j,i}^{t}  {u}^{t+\frac{1}{2}}_{i}$ and bias information $p_{j,i}^{t}  {\mu}^{t}_{i}$ to clients $ j\in \mathcal{N}^{out}_{i,t}$.
        
        Pull weights $p_{i,j}^{t}  {u}^{t+\frac{1}{2}}_{j}$ and bias information $p_{i,j}^{t}  {\mu}^{t}_{j}$ from clients $ j  \in \mathcal{N}^{in}_{i,t}$.

        P2P updating by $ {u}^{t+1}_{i} = \sum_{j \in \mathcal{N}^{in}_{i,t} } p_{i,j}^{t}  {u}^{t+\frac{1}{2}}_{j} $ and ${\mu}^{t+1}_{i} = \sum_{j \in \mathcal{N}^{in}_{i,t} } p_{i,j}^{t}  {\mu}^{t}_{j} $.
        
        De-bias the updated model by  $z_{i}^{t+1} = u_{i}^{t+1}/\mu^{t+1}_{i} $.
     }
}
\end{algorithm}

%% file: content/06_theory.tex
\section{Theoretical Analysis}\label{theory}

In this section, we provide a detailed convergence theorem for the proposed algorithm DFedPGP and explore how the partial personalization and gradient push work. The detailed derivation process can be found in the appendix \ref{ap:proof}. %Below, we make the following assumptions first.

\subsection{Assumption}

\begin{assumption}[$\mathcal{B}$-bounded Connectivity \citep{chen2023enhancing}\label{assmp:mixing connectivity}]
The time-varying graph (i.e., the communication topology) is B-bounded strongly connected to ensure the convergence of model training \citep{assran2019stochastic}. There exists a window size  $\mathcal{B} \geq 1$ 1 such that the graph union $ \bigcup_{k=l}^{l+ \mathcal{B}-1} \mathcal{G}(k)(l=0,1,2, \cdots)$ is strongly connected. Note that if $ \mathcal{B} = 1$, each instance of graph $ \mathcal{G}(k)$ is strongly connected at global iteration $ k $.
\end{assumption}

\begin{assumption}[Smoothness \citep{pillutla2022federated}]\label{assmp:smoothness}
For each client $i=\{1,\ldots,m\}$, the function $F_i$ is continuously differentiable.  
There exist constants $L_u, L_v, L_{uv}, L_{vu}$ such that
for each client $i=\{1,\ldots,m\}$:
\begin{itemize}[topsep=0pt,itemsep=0pt,leftmargin=\widthof{(a)}]
\item $\nabla_u F_i(u_i,v_i)$ is $L_u$--Lipschitz with respect to~$u_i$ and $L_{uv}$--Lipschitz with respect to~$v_i$
\item $\nabla_v F_i(u_i,v_i)$ is $L_v$--Lipschitz with respect to~$v_i$ and $L_{vu}$--Lipschitz with respect to~$u_i$.
\end{itemize}
 
We summarize the relative cross-sensitivity of $\nabla_u F_i$ with respect to~$v_i$ and $\nabla_v F_i$ with respect to~$u$ with 
the scalar
\begin{equation}\label{eqn:chi-def}
\small
\chi := \max\{L_{uv},\,L_{vu}\}\big/\sqrt{L_u L_v}.\nonumber
\end{equation}
\end{assumption}

\begin{assumption}[Bounded Variance \citep{pillutla2022federated}] \label{assmp:stoc-grad-var}
The stochastic gradients in Algorithm \ref{DFedPGP} have bounded variance. That is, for all $u_i$ and $v_i$, 
% \begin{align*}
%     \E\bigl[ \widetilde\nabla_u F_i(u_i, v_i)\bigr] &= \nabla_u F_i(u_i, v_i),
%     %
%     \\
%     \E\bigl[ \widetilde\nabla_v F_i (u_i, v_i)\bigr] &= \nabla_v F_i(u_i, v_i)\,.
% \end{align*}
there exist constants $\sigma_u$ and $\sigma_v$ such that
\begin{align*} 
\small
    \E\bigl[\bigl\|\nabla_u F_i(u_i, v_i; \xi_i) - \nabla_u F_i(u_i, v_i)\bigr\|^2\bigr] &\le \sigma_u^2, \\
    \E\bigl[\bigl\|\nabla_v F_i(u_i, v_i; \xi_i) - \nabla_v F_i(u_i, v_i)\bigr\|^2\bigr] &\le \sigma_v^2.
\end{align*}
\end{assumption}

\begin{assumption}[Partial Gradient Diversity \citep{pillutla2022federated}]
\label{assmp:grad-diversity.}
There exist a constant $\sigma^2_g$ such that
\[
\small
\textstyle
    \|\nabla_u F_i(u_i, v_i) - \nabla_u F(u_i, V)\bigr\|^2
    \le \sigma_g^2,~ \forall u_i,~V.
\]
\end{assumption}

%\begin{assumption}[Bounded Gradient Dissimilarity \citep{karimireddy2020scaffold}]
%\label{assmp:bounded-gradient-dissimilarity.}
%\textcolor{blue}{We have $\frac{1}{m}\sum_{i=1}^m \| %\nabla_uF_i(u_i, V) \|^2 \leq G^2 + B^2 \| \nabla_u F(\bar{u}, V) \|$.} 
%\end{assumption}

Assumption \ref{assmp:mixing connectivity} is commonly adopted in gradient push work \citep{chen2023enhancing,li2023asymmetrically,nedic2014distributed}: it is considerably weaker than requiring each $\mathcal{G}(t)$ be connected for it allows each client to connect in time-varying and directed topologies.
Assumption \ref{assmp:smoothness}--\ref{assmp:grad-diversity.} are mild and commonly used in characterizing the convergence rate of FL \citep{li2023dfedadmm,shi2023towards,karimireddy2020scaffold,sun2022decentralized}.

\subsection{Challenge and Proof}

\textbf{Challenges of Convergence Analysis.} 
(1) Compared to the classical Push-sum based method like SGP\citep{assran2019stochastic}, $u_{i}^{t,k}-u_{i}^{t,0}$ after multiple local iterations and alternately update is not an unbiased estimate of $\nabla F_i(u_i^t)$. Multiple local iteration analyses are non-trivial; 
(2) In contrast to the symmetric topology, DFedPGP communicates with its clients based on an asymmetric topology, resulting in $\sum_{j}p_{i,j}^t \neq 1$. As a consequence, each client needs to maintain a Push-sum weight $\mu_i^t$ to de-bias the model parameters;
(3) To realize better-personalized performance, we need to analyze the convergence in a partial model personalized way, where the shared part $u$ is updated with gradient pushing and pulling while the personalized part $v$ is updated with local SGD separately. Now, we present the rigorous convergence analysis of DFedPGP as follows.

% (2) Compared to the symmetric topology, ournot trivial; the proposed algorithm utilizes an asymmetric topology, which results in $\sum_{j}p_{i,j}^t$ not being equal to 1. As a result, each client needs to maintain \textsc{PushSum} weight $w_i^t$ to de-biased model parameters. 

\begin{theorem}\label{th:theorem}
Under Assumptions 1-5, the local learning rates satisfy $0<\eta_u<\frac{\delta}{4\sqrt{2}L_uK_u}$, $F^{*}$ is denoted as the minimal value of $F$, i.e., $F(\bar{u}, V)\ge F^*$ for all $\bar{u} \in \mathbb{R}^{d}$, and $V=(v_1,\ldots,v_m)\in\mathbb{R}^{d_1+\ldots+d_m}$. Let $\bar{u}^t = \frac{1}{m}\sum_{i=1}^m u^t_i$ and denote $\Delta_{\bar{u}}^t$ and $\Delta_{v}^t$ as:
\begin{center}
\small
    $\Delta_{\bar{u}}^t = \bigl \| \nabla_u F(\bar{u}^t, V^{t+1}) \bigr\|^2 \,, ~~~
        \Delta_v^t = \frac{1}{m}\sum_{i=1}^m \bigl\|\nabla_v F_i(u_i^t, v_i^t)\bigr\|^2 \,.$
\end{center}
Therefore, we have the convergence analysis below:
\begin{equation}\label{th:theorem_dfeddgp}
\small
    \begin{split}
        &\frac{1}{T}\sum_{i=1}^T \bigl(\frac{1}{L_u} \E \bigl[\Delta_{\bar{u}}^t \bigr] + \frac{1}{L_v} \E [\Delta_{v}^t \bigr] \bigr) 
        \leq \mathcal{O}\Big(\frac{F(\bar{u}^1, V^1) - F^*}{\sqrt{T}}  \\
        & +\frac{ (1+L_v)\sigma_v^2 }{\sqrt{T}} 
        + \big(\sigma_u^2 + \sigma_g^2 \big)
        \big( \frac{ C^2}{(1-q)^2 L_u T} 
        +  \frac{1}{K_u L_u \sqrt{T}} \\
        &   + \frac{1}{K_u L_u \delta^2 T^{3/2}}
        + \frac{ C^2}{K_u L_u (1-q)^2 T^{3/2}}
        +\frac{L_{vu}^2 C^2 }{(1-q)^2 L_u^2 \sqrt{T}} \big).
    \end{split}
\end{equation}
\end{theorem}
The parameters $C$ and $q$ are related to the communication topology as \citep[Lemma 3]{assran2019stochastic}. $\delta$ is the minimum sum of any row elements in the matrix $\prod_{i=1}^t \mathcal{G}(i)$ for all $t\geq 0$ as \citep[Proposition 2.1]{taheri2020quantized}. With the proper step sizes, we have the following corollary.

\begin{corollary}[Convergence Rate for DFedPGP]\label{th:corollary}
    Under Theorem \ref{th:theorem} and by setting the local learning rates 
     $\eta_u=\mathcal{O}({1}/{L_uK_u\sqrt{T}}), \eta_v=\mathcal{O}({1}/{L_vK_v\sqrt{T}})$, it holds that: 
    \begin{equation} %\label{th:corollary_dfeddgp}  
    \small
    \begin{split}
    & \frac{1}{T}\sum_{i=1}^T \bigl(\frac{1}{L_u} \E \bigl[\Delta_{\bar{u}}^t \bigr] 
    +  \frac{1}{L_v} \E [\Delta_{v}^t \bigr] \bigr) \\
    & \leq \mathcal{O}\Big(\frac{F(\bar{u}^1, V^1) - F^* +\sigma_1^2}{\sqrt{T}} 
    + \frac{\sigma_2^2}{T} 
    + \frac{\sigma_3^2}{\sqrt{T^3}} \Big),
    \end{split}
    \end{equation}
    where     
    \begin{equation}         \small
    \begin{split}
        \sigma_1^2 & = (1+L_v)\sigma_v^2
        + \big(\frac{1}{K_u L_u} 
        + \frac{ L_v \chi^2 C^2}{(1-q)^2 L_u} \big) 
        \big(\sigma_u^2 + \sigma_g^2\big) , \\
        \sigma_2^2 &= \frac{C^2}{(1-q)^2 L_u  }
         \big(\sigma_u^2 + \sigma_g^2\big) , \\
        \sigma_3^2 &= \big( \frac{1}{K_u L_u \delta^2}
        + \frac{ C^2}{(1-q)^2 K_u L_u} \big) \big(\sigma_u^2 + \sigma_g^2 \big). 
    \end{split}
    \end{equation}
\end{corollary}

\begin{table*}[ht]
\centering
\small
%\vspace{-0.4cm}
\caption{ \small  Test accuracy (\%) on CIFAR-10 \& 100 in both Dirichlet and Pathological distribution settings.}
\label{ta:all_baselines}
\begin{tabular}{lcccc|cccc} 
\toprule
\multirow{3}{*}{Algorithm} & \multicolumn{4}{c|}{CIFAR-10}                                                                                 & \multicolumn{4}{c}{CIFAR-100}                                                                             \\ 
\cmidrule{2-9}
                           & \multicolumn{2}{c}{Dirichlet}                            & \multicolumn{2}{c|}{Pathological}                  & \multicolumn{2}{c}{Dirichlet}                       & \multicolumn{2}{c}{Pathological}                    \\ 
\cmidrule{2-9}
                           & $\alpha$ = 0.1               & $\alpha$ = 0.3            & c = 2                    & c = 5                   & $\alpha$ = 0.1          & $\alpha$ = 0.3            & c = 5                   & c = 10                    \\ 
\midrule
Local                      & $78.96_\pm.42$               & $63.20_\pm.28$            & $85.16_\pm.18 $          & $68.56_\pm.35$          & $39.38_\pm.33$          & $22.59_\pm.49$            & $ 71.34_\pm.46 $        & $53.15_\pm.31 $           \\
FedAvg                     & $84.17_\pm.28$               & $79.66_\pm.20$            & $85.04_\pm.11$           & $81.18_\pm.27$          & $57.43_\pm.03$          & $55.06_\pm.06 $           & $ 69.05_\pm.43 $        & $66.37_\pm.48 $           \\
FedPer                     & $88.57_\pm.09$               & $84.06_\pm.29$            & $90.94_\pm.24$           & $86.97_\pm.35$ & $54.23_\pm.14$          & $34.07_\pm.76$            & $ 78.48_\pm.93 $        & $ 70.38_\pm.02 $          \\
FedRep                     & $88.78_\pm.40$               & $84.50_\pm.05$            & $91.09_\pm.12 $          & $ 86.22_\pm.51$         & $ 44.02_\pm.98 $          & $26.88_\pm.49$            & $78.77_\pm.19$          & $67.65_\pm.43 $           \\
FedBABU                    & $87.79_\pm.53$               & $83.26_\pm.09$            & $\textbf{91.28}_\pm.15 $ & $ 83.90_\pm.24 $        & $60.23_\pm.07 $         & $52.37_\pm.82$            & $ 77.50_\pm.33 $        & $69.81_\pm.12 $           \\
Ditto                      & $80.22_\pm.10$               & $73.51_\pm.04$            & $84.96_\pm.40 $          & $ 75.59_\pm.32 $        & $48.85_\pm.54 $         & $48.65_\pm.50 $           & $69.48_\pm.45 $         & $ 60.77_\pm.30 $          \\ 
\midrule
DFedAvgM                   & $86.94_\pm.62$               & $82.49_\pm.57$            & $90.23_\pm.97$           & $85.26_\pm.47$          & $58.80_\pm.82 $         & $54.89_\pm.77$            & $75.89_\pm.65$          & $70.55_\pm.44$            \\
OSGP                       & $87.39_\pm.13$               & $ 83.14_\pm.18 $          & $90.72_\pm.08 $          & $84.69_\pm.25 $         &$59.76_\pm.69 $         & $ 54.98_\pm.48 $          & $76.70_\pm.59 $         & $ 71.08_\pm.52 $          \\
Dis-PFL                    & $87.77_\pm.46 $              & $82.71_\pm.28$            & $88.19_\pm.47$           & $84.18_\pm.61$          &$56.06_\pm.20$           & $46.65_\pm.18 $           & $71.79_\pm.42 $         & $65.35_\pm.10 $           \\ 
\midrule
DFedPGP                    & $\textbf{88.85}_\pm.21$      & $ \textbf{85.61}_\pm.05$  & $91.26_\pm.05$           & $\textbf{87.12}_\pm.37$          & $\textbf{66.26}_\pm.25$    & $\textbf{ 57.66}_\pm.42$     & $\textbf{78.78}_\pm.41$   & $\textbf{72.19}_\pm.21$  \\
\bottomrule
\end{tabular}
\vspace{-0.3cm}
\end{table*}

\begin{remark}
    Corollary \ref{th:corollary} provides explicit insights into how various key factors affect the convergence of DFedPGP. Specifically, the convergence analysis illustrates that the large values of the gradient variance $\sigma_u^2$, $\sigma_v^2$, $\sigma_g^2$ and gradient bounded $B$ lead to slower convergence.  It also shows that more local update steps $K_u$ accelerate the convergence, quantitatively justifying the benefit of exploiting more local updates in the algorithm. Also, the smoothness of local loss functions such as $L_u$, $L_v$, and $L_{vu}$, have a significant influence on the convergence bound.
\end{remark}

\begin{remark}
    As the definition in \citep{assran2019stochastic}, the $q$ in Corollary \ref{th:corollary} can be explicitly expressed as $q = (1-a^{\Delta \mathcal{B}})^{\frac{1}{\Delta \mathcal{B}+1}}$, where $\Delta$ is the diameter of the communication network, $\mathcal{B}$ is the same defined in Assumption \ref{assmp:mixing connectivity}, and $a<1$ is a constant. Note that the bound will be tighter as $q$ decreases, which means the network connectivity improves and clients exchange parameters with more neighbors in the communication progress. Moreover, the connectivity constant $C$ decreases as the connectivity of the communication network improves, which leads to the same conclusion as $q$. More details about the communication constant can be found in Lemma \ref{th:lemma3}.
\end{remark}

\begin{remark}
    From Corollary \ref{th:corollary}, the proposed DFedPGP has a convergence rate of $\mathcal{O}\left(\frac{1}{\sqrt{T}}\right)$.  
    %$\mathcal{O}(1/\sqrt{T})$. 
    This result is consistent with the convergence rate achieved by \citep{pillutla2022federated,shi2023towards} in PFL. Moreover, when the smoothness of the shared parameters is not good, it means that $L_u$ is large, the term $\mathcal{O}(\frac{1}{T}+\frac{1}{\sqrt{T^3}})$ can be neglected compared to $\mathcal{O}(\frac{1}{\sqrt{T}})$. 

\end{remark}

%% file: content/07_experiment.tex
\section{Experiments}\label{exper}

In this section, we conduct extensive experiments to verify the effectiveness of the proposed DFedPGP in data heterogeneity and computation resources heterogeneity scenarios. Below, we first introduce the experimental setup.

\subsection{Experiment Setup}

% \textbf{Dataset and Data Partition.}
\textbf{Dataset and Data Partition.} We evaluate the performance on CIFAR-10, CIFAR-100  \citep{krizhevsky2009learning}, and Tiny-ImageNet  \citep{le2015tiny} datasets in the Dirichlet distribution and Pathological distribution, where CIFAR-10 and CIFAR-100 are two real-life image classification datasets with total 10 and 100 classes. Experiments on the Tiny-ImageNet dataset are placed in \textbf{Appendix} \ref{exper:tiny} due to the limited space. We partition the training and testing data according to the same Dirichlet distribution Dir($\alpha$) such as $\alpha =0.1$ and $\alpha =0.3$ for each client. The smaller the $\alpha$ is, the more heterogeneous the setting is. Meanwhile, for each client, we sample 2 and 5 classes from a total of 10 classes on CIFAR-10, and 5 and 10 classes from a total of 100 classes on CIFAR-100, respectively \citep{zhang2020personalized}. The number of sampling classes is represented as “c” in Table \ref{ta:all_baselines} and the fewer classes each client owns, the more heterogeneous the setting is.

\begin{figure*}[h]
    \centering
     %\vspace{-0.4cm}
    \begin{subfigure}{1\linewidth}
		\centering		\includegraphics[width=1.0\linewidth]{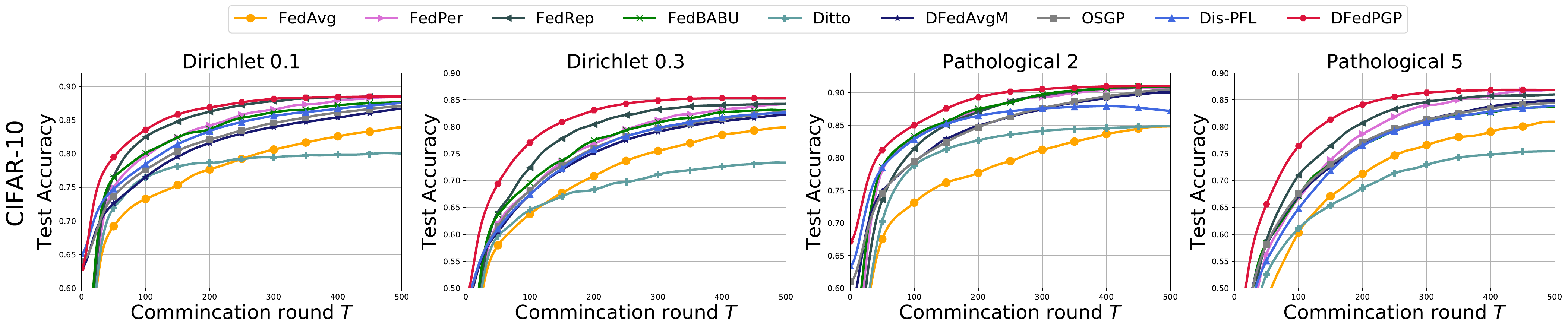}
		%\caption{\small CIFAR-10}
	    \label{CIFAR-10}
     \vspace{-0.3cm}
	\end{subfigure}
         \begin{subfigure}{1\linewidth}
		\centering		\includegraphics[width=1.0\linewidth]{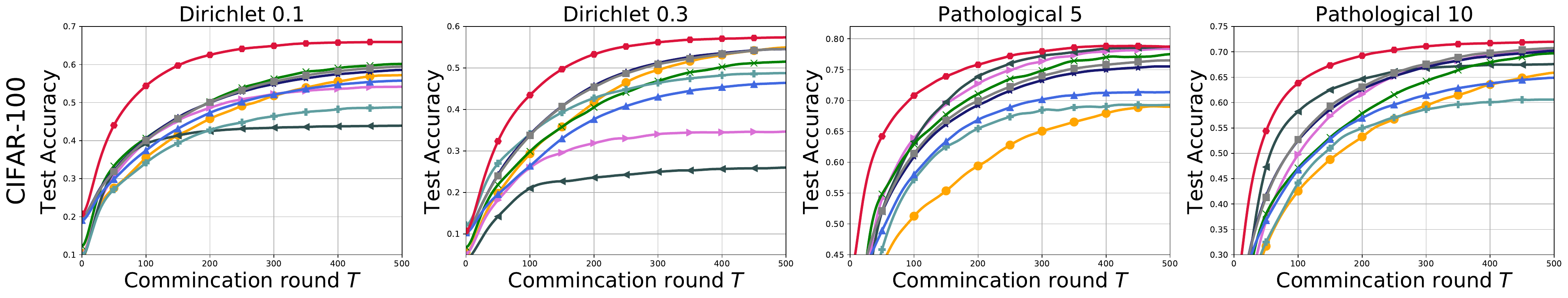}
		%\caption{\small CIFAR-100}
	    \label{CIFAR-100}
        \vspace{-0.3cm}
	\end{subfigure}
        \caption{\small Test accuracy on CIFAR-10 (first line) and CIFAR-100 (second line) with heterogenous data partitions. With limited pages, we only show the training progress of the typical methods.  }
	\label{fig:baseline}
     \vspace{-0.2cm}
\end{figure*}

\medskip
\noindent
\textbf{Baselines and Backbone.}
We compare the proposed methods with the SOTA baselines PFL. For instance, Local is the simplest method where each client only conducts training on their own data without communicating with other clients. Federated learning methods include FedAvg  \citep{mcmahan2017communication}, FedPer  \citep{arivazhagan2019federated}, FedRep  \citep{collins2021exploiting}, FedBABU \citep{oh2021fedbabu} and Ditto \citep{li2021ditto}. For DFL methods, we take DFedAvgM \citep{sun2022decentralized}, Dis-PFL  \citep{Rong2022DisPFL} and OSGP \citep{assran2019stochastic} as our baselines. All methods are evaluated on ResNet-18  \citep{he2016deep} and replace the batch normalization with the group normalization followed by  \citep{wu2018group,Rong2022DisPFL,shi2023improving} to avoid unstable performance. For the partial PFL methods, we set the lower linear classifier layers as the personal part responsible for complex pattern recognition, and the rest upper representation layer as the shared layers focusing on feature extraction. Note that we compare the personal test accuracy for all methods since our goal is to solve PFL.

\medskip
\noindent
\textbf{Implementation Details.}
We keep the same experiment setting for all baselines and perform 500 communication rounds with 100 clients. The client sampling radio is 0.1 in CFL, while each client communicates with 10 neighbors in PFL accordingly. The batch size is  128. For DFedPGP, we train the shared part for 5 epochs per round as the same as other baselines, and train 1 epoch for the personal part to align the shared part and save the computation consumption. We set SGD \citep{robbins1951a} as the base optimizer for all methods with a learning rate $\eta_u = 0.1$ to update the model parameters and the learning rate decreasing by 0.99× exponentially. All methods are set with a decay rate of 0.005 and a local momentum of $0.9$. We report the mean performance with 3 different random seeds and more details of the baseline methods can be found in \textbf{Appendix} \ref{exp:baselines}.

\subsection{Performance Evaluation}
\textbf{Comparison with the Baselines.}
  As shown in Table 1 and Figure 2, the proposed DFedPGP outperforms other baseline methods with the best stability and better performance in both different datasets and different data heterogeneity scenarios. Specifically, on the CIFAR-10 dataset, DFedPGP achieves 86.50\% on the Directlet-0.3 setups, 1.11\%  ahead of the best-comparing method FedRep. On the CIFAR-100 dataset, DFedPGP achieves at least 2.60\% and 1.11\% improvement from the other baselines on the Directlet-0.3 and Pathological-10 settings. The communication based on the directed graph allows clients to choose their in-neighbors and out-neighbors flexibly, guaranteeing that they can choose useful information for their local training.

\medskip
 \noindent 
\textbf{Comparison on Heterogeneous Setting.} 
We discuss two data heterogeneities, Dirichlet distribution and Pathological distribution in Table \ref{ta:all_baselines}, and prove the effectiveness and robustness of the DFedPGP. In PFL tasks, since the local training can’t cater for all classes inside clients, the accuracy decreases with the heterogeneity decreasing. \footnote{Generally, the higher data heterogeneity means a greater difference between local data distribution.  But in extreme data heterogeneity PFL tasks, the higher heterogeneity means it owns fewer data classes locally, which makes the classification task easier and clients will achieve better performance. For example, in the Pat-2 setting, the local binary classification task is easier than the five classification tasks in the Pat-5 setting, so the average test performance in Pat-2 is better than that in Pat-5. The same phenomenon can be seen in most PFL works \citep{Rong2022DisPFL,oh2021fedbabu,xu2023personalized,zhang2020personalized,huang2021personalized}.} 
On the CIFAR-10 dataset, when the heterogeneity decreases from 0.1 to 0.3 in Directlet distribution, FedRep drops from 88.78\% to 84.50\%, while DFedPGP drops about 3.24\% to 85.61\%, meaning its stronger stability for several heterogeneous settings. On the Pathological distribution, DFedPGP beat the best-compared baselines over 1.11\% on the CIFAR-100 dataset with only 10 categories per client, which confirm that the proposed methods could achieve better performance in the strong heterogeneity.

\medskip
\noindent
\textbf{Comparison on the Convergence Speed with Baselines.} We show the convergence speed via the learning curves of the compared methods in Figure \ref{fig:baseline} and Table \ref{ta:convergency speed}. DFedPGP achieves the fastest convergence speed among the comparison methods, which benefits from the direct partial model transmission and alternate update. For example, DFedPGP is almost twice as fast as the other methods in Dirichlet-0.3 on CIFAR-10 and CIFAR-100 settings. In comparison with the CFL methods, directly learning the neighbors’ feature representation in DFL can speed up the convergence rate for personalized problems. Notably, we target the setting where the busiest node’s communication bandwidth is restricted for fairness when compared with the CFL methods.

\begin{table} [b]
\centering\small
\caption{ \small The required communication rounds when achieving the target accuracy (\%).}
\label{ta:convergency speed}
\begin{tabular}{lcc|cc} 
\toprule
\multirow{3}{*}{Algorithm} & \multicolumn{2}{c|}{CIFAR-10}   & \multicolumn{2}{c}{CIFAR-100}  \\ 
\cmidrule{2-5}
                           & Dir-0.3        & Pat-2          & Dir-0.3      & Pat-10          \\ 
\cmidrule{2-5}
                           & acc@80         & acc@90         & acc@45       & acc@65          \\ 
\midrule
FedAvg                     & -              & -              & 234          & 456             \\
FedPer                     & 262            & 343            & -            & 246             \\
FedRep                     & 189            & 322            & -            & 210             \\
FedBABU                    & 270            & 312            & 261          & 314             \\
Ditto                      & -              & -              & 279          & -               \\ 
\midrule
DFedAvgM                   & 320            & 452            & 187          & 249             \\
OSGP                       & 309            & 439            & 192          & 230             \\
Dis-PFL                    & 307            & -              & 368          & 492             \\ 
\midrule
DFedPGP                    & \textbf{ 131 } & \textbf{ 224 } & \textbf{111} & \textbf{ 113 }  \\
\bottomrule
\end{tabular}
\end{table}

\begin{table*} [th]
\centering \small 
\vspace{-0.4cm}
\caption{ \small Test accuracy (\%) in computation resources heterogeneous setting.}
\label{ta:resource}
\begin{tabular}{l|ccccccccc} 
\toprule
Algorithm & FedAvg & FedPer & FedRep & FedBABU & Ditto & DFedAvgM & Dis-PFL & OSGP  & DFedPGP  \\ 
\midrule
Dir   & 75.76  & 81.06  & 83.08  & 72.66   & 75.63 & 82.70    & 82.41   & 82.81 & \textbf{83.63 }   \\
Pat     & 81.04  & 91.09  & 88.57  & 83.06  & 82.26 & 91.52    & 91.40   & 91.61 & \textbf{91.83}    \\
\bottomrule
\end{tabular}
\end{table*}

\medskip
\noindent
\textbf{Comparison on Computation Resources Heterogeneity.} In reality, the federating process often involves heterogeneous devices, which means the shared parameters are trained at different convergence levels.  We follow \citep{avdiukhin2021federated, Rong2022DisPFL} to divide 100 clients into 5 parts and transmit their parameters after 1, 2, 3, 4 and 5 local epochs to simulate the different computation capabilities of each device \citep{wang2020tackling}.

\begin{figure}[t]
	\centering
	\begin{minipage}{0.495\linewidth}
		\centering
		\includegraphics[width=0.9\linewidth]{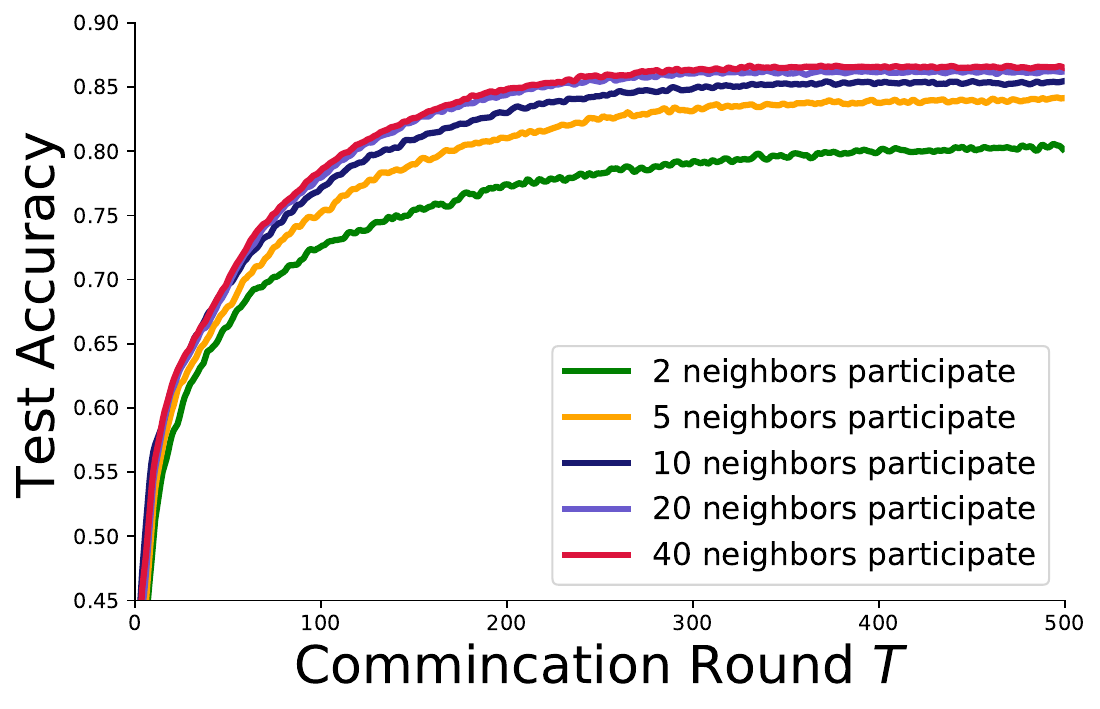}
		\subcaption{}
		\label{fig:abla_neighbor}%文中引用该图片代号
	\end{minipage}
	%\qquad
	\begin{minipage}{0.495\linewidth}
		\centering
		\includegraphics[width=0.9\linewidth]{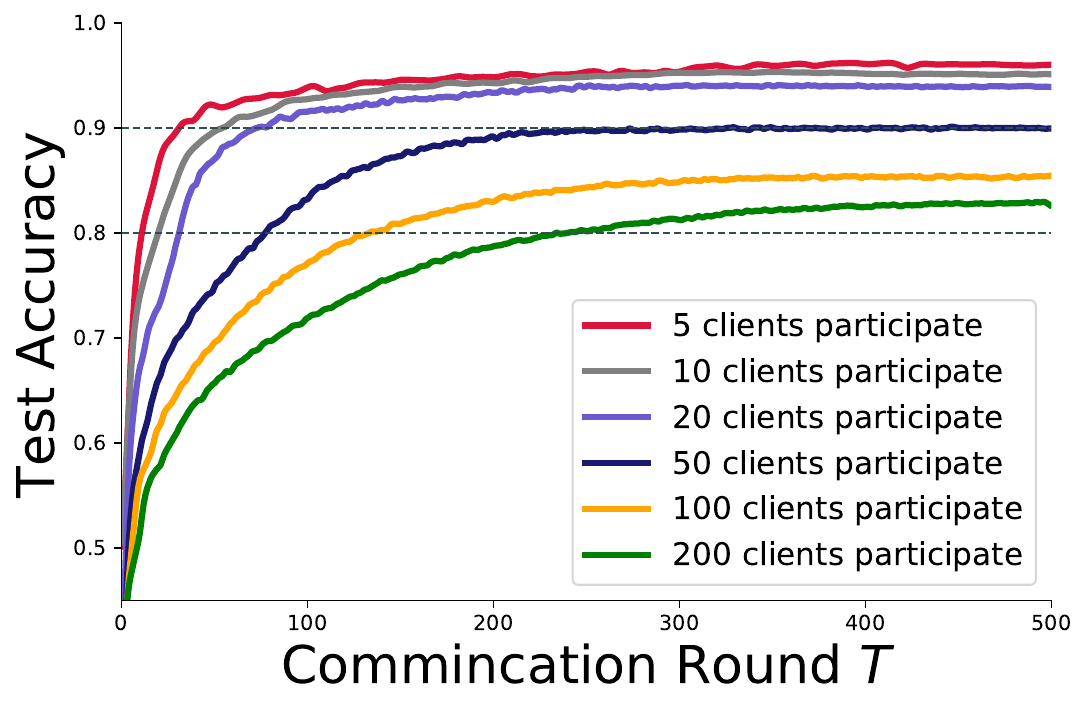}
		\subcaption{ }
		\label{fig:abla_participated}
	\end{minipage}
 \caption{ Ablation study. (a) Effect of the number of neighboring clients. (b) Effect of the number of participated clients.}

\end{figure}

Table \ref{ta:resource} shows the comparison among PFL methods under a computation heterogeneous setting on the CIFAR10 dataset with Dirichlet-0.3 distribution. DFedPGP achieves the best compared with the other baselines, indicating that the partial gradient push can alleviate the effect of the different convergence level aggregation. Another interesting finding is that FedRep is the best PFL method encountering the computation heterogeneous challenge in CFL, indicating that keeping the classifiers locally and updating the private and shared parts alternately is an effective way to solve the computation heterogeneity problem.

\subsection{Ablation Study}

\textbf{Number of Neighboring Clients.}
We conduct experiments on the convergence performance under different neighbor participation numbers of \{2, 5, 10, 20, 40\} on CIFAR-10 with Dir-0.3 distribution. As shown in Figure \ref{fig:abla_neighbor}, the highest personalized performance is achieved when the participation number is set to 40, which indicates that with more clients exchanging their information, a quicker convergence speed will be achieved, aligning with our insight. Notably, the proposed DFedPGP can realize a stable convergence even when transmitting information to only 2 neighbors.

\medskip
\noindent
\textbf{Number of Participated Clients.} As depicted in Figure \ref{fig:abla_participated}, we compare the personalized performance between different numbers of client participation of \{5, 10, 20, 50, 100, 200\} on the CIFAR-10 dataset with Dirichlet-0.3 distribution under the same hyper-parameters. Compared with larger participated clients \{50, 100, 200\}, the smaller participated clients \{5, 10\} can achieve better test accuracy and convergence as the number of local data increases.

%It is clearly seen that the test performance will get a great margin with the participation of clients decreasing, which means that the more training data each client owns, the better performance it will achieve. 

\medskip
\noindent
\textbf{Module Augmentation Ablation.}  We investigate the effect of partial personalization and directed communication with different data heterogeneity on the CIFAR-10 dataset. From Table \ref{ta:module_abla}, DFedPGP achieves the best in both Dirichlet-0.3 and Pathological-2. In the comparison of partial personalization, DFedAvgM-P and DfedPGP outperform their full personalization versions DFedAvgM and OSGP by a great margin separately. In directed communication comparison, DfedPGP and OSGP outperform their undirected versions DFedAvgM-P and DFedAvgM, respectively. From the ablation study, both partial personalization and directed communication have a great influence on decentralized and personalized performance. %Combining the results in Table \ref{ta:all_baselines}, OSGP is slightly better than DFedAvgM, especially in the more difficult task (CIFAR-100 dataset). This advises us to explore more directed partial model collaboration between clients for DPFL in the future.
Randomly choosing clients' in-neighbors and out-neighbors in directed graphs means that the shared part model has a larger feature search space among clients, compared with the undirected graphs. Intuitively, this increases the involved clients in one communication round and enhances communication efficiency.

\begin{table} [t]
\centering \scriptsize
\caption{ \small Test accuracy (\%) of different module augmentation.}
\label{ta:module_abla}
\begin{tabular}{l|cc|c|l} 
\toprule
Algorithm  & \begin{tabular}[c]{@{}c@{}}Partial\\~Personalization\end{tabular} & \begin{tabular}[c]{@{}c@{}}Directed \\Communication\end{tabular} & Dir   & Pat    \\ 
\midrule
DFedAvgM   &                                                                   &                                                                  & 82.49 & 90.23  \\
DFedAvgM-P &     \checkmark                                                               &                                                                  & 84.69 & 90.90  \\
OSGP       &                                                                   &  \checkmark                                                                 & 83.14 & 90.72  \\
DFedPGP    &   \checkmark                                                                 &    \checkmark                                                   & \textbf{85.61} & \textbf{91.26}  \\
\bottomrule
\end{tabular}
\end{table}

%% file: content/08_conclusion.tex
\section{Conclusion}

In this paper, we propose a novel method DFedPGP for PFL, which simultaneously guarantees robust communication and better personalized performance with convergence guarantee via partial gradient push in a directed communication graph. The directed collaboration allows clients to choose their corporate neighbors flexibly, which guarantees effective aggregation and learning under data and device heterogeneity scenarios.
For theoretical findings, we present the personalized convergence rate of $\mathcal{O}(1/\sqrt{T})$ in the non-convex setting for DFedPGP. Empirical results also verify the superiority of the proposed approach.

\medskip
\noindent
\textbf{Future Works.} In the current work, we mainly focus on the theoretical analysis and the experiment verification of the partial Push-sum based optimization framework with the directed graph for DPFL. It can be extended with effective client selection methods to speed up the convergence and improve personalized performance in the future. %And we will perform continuous research in this aspect in future works.

\medskip
\noindent
\textbf{Acknowledgment.} 
This work is supported by STI 2030—Major Projects (No. 2021ZD0201405).

%% file: content/09_appendix.tex
\newpage
\onecolumn 

\vspace{0.5in}
\begin{center}
 \rule{6.875in}{0.7pt}\\ % 4.0
 {\Large\bf Supplementary Material for\\ `` Decentralized Directed Collaboration for Personalized Federated Learning ''}
 \rule{6.875in}{0.7pt}
\end{center}
\appendix

\noindent
In this part, we provide supplementary materials including more introduction to the related works, experimental details and results, and the proof of the main theorem.

\begin{itemize}
    \item \textbf{Appendix} \ref{ap:related_works}: More details in the related works.
    \item \textbf{Appendix} \ref{ap:selection}: More details in the client selection.
    \item \textbf{Appendix} \ref{ap:baseline}: More details in the experiments.
    \item \textbf{Appendix} \ref{ap:proof}: Proof of the theoretical analysis.
\end{itemize}

\section{More Details in the Related Works}\label{ap:related_works}

\textbf{Decentralized/Distributed Training.}
Decentralized/Distributed Training, which allows edge clients to communicate with each other in a peer-to-peer manner, is an encouraging field that shares several benefits: (1) guarantees collaborative learning through local computation and the exchange of model parameters; (2) is low for feeding the models of adjacent clients, generating a more intelligent private model; (3) avoids central failure in the collaborative system. Thus, Decentralized/Distributed Training has been applied in many fields\citep{beltran2022decentralized}: (1) Healthcare \citep{nguyen2022novel}, favoring the decentralization of clinical records and collaborative diagnosis; (2) Mobile Services \citep{wang2022accelerating}, decreasing response times and increasing the bandwidth of constraints devices; (3) Vehicles \citep{yu2020proactive}, ensuring high mobility and local storage management. 

Since the prototype of DFL (fully decentralized federated learning \citep{lalitha2018fully}) was proposed, it has been a promising approach to save communication costs as the compromise of CFL. By combining SGD and gossip, early work achieved decentralized training and convergence in \cite{blot2016gossip}. D-PSGD \cite{lian2017can} is the classic decentralized parallel SGD method. FastMix \cite{ye2020decentralized} investigates the advantage of increasing the frequency of local communications within a network topology, which establishes the optimal computational complexity and near-optimal communication complexity. DeEPCA \cite{ye2021deepca} integrates FastMix into a decentralized PCA algorithm to accelerate the training process. DeLi-CoCo \cite{Hashemi2022On} performs multiple compression gossip steps in each iteration for fast convergence with arbitrary communication compression. Network-DANE \cite{li2020communication} uses multiple gossip steps and generalizes DANE to decentralized scenarios.
QG-DSGDm \cite{lin2021quasi} modifies the momentum term of decentralized SGD (DSGD) to be adaptive to heterogeneous data, while SkewScout \cite{hsieh2020non} replaces batch norm with layer norm. Meta-L2C \cite{li2022learning} dynamically updates the mixing weights based on meta-learning and learns a sparse topology to reduce communication costs. The work in \cite{zhu2022topology} provides the topology-aware generalization analysis for DSGD, they explore the impact of various communication topologies on the generalizability.

\section{More details in the client selection}\label{ap:selection}

\noindent
\textbf{Push sum based directed distributed averaging.}  The initial Push sum algorithm \citep{kempe2003gossip} considers the averaged consensus $1/n \sum^n_{i=1}  y_i^{0}$ of all clients. Let $y_i^{0} \in \mathbb{R}^d$ be a vector at client $i$ and typical gossip iterations forms $y_i^{t+1}=\sum_{j=1}^n p_{i, j}^{t} y_j^{t}$,  where $P^{t} \in \mathbb{R}^{n\times n}$ is the mixing matrix. Inspired by the Markov chains \citep{seneta2006non}, the mixing matrices $P^t$ are designed to be column stochastic (each column must sum to 1). So the gossip iterations converge to a limit $y_i^{\infty} =  \pi_i \sum_{j=1}^n y_j^0$, where \textbf{$\pi$} is the ergodic limit of the chain. When the matrices $P^{t}$ are symmetric, it is straightforward to satisfy $\pi_i = 1/n$ by defining $P^{t}$ doubly-stochastic (each row and each column must sum to 1). However, symmetric $P^{t}$ are hard to meet due to the unstable communication in reality. The Push sum algorithm adds one additional scalar parameter $w^t_i$ to achieve $\pi_i = 1/n$ under the column-stochastic and asymmetric mixing matrices $P^{t}$. The parameter is initialized to $w^0_i = 1$ for all $i$ and updated using the same linear iteration, $w_i^{t+1} = \sum^n_{j=1} p_{i,j}^tw_j^t$. It recovers the average of the initial vectors by computing the de-biased ratio $y^\infty_i / w^\infty_i$, and the scalar parameters converge to $w_i^\infty = \pi_i\sum_{j=1}^n w^0_j$. 

\noindent
\textbf{Directed random graph.} We transfer the mixing matrices from column stochastic (all columns sum to 1) to row stochastic (all rows sum to 1), meaning that the clients can actively select the information they need rather than passively accept, which is more beneficial for directed collaboration in the DPFL problem. 
In the experiments, each client pulls the shared parameters from its in-neighbors $j \in \mathcal{N}^{in}_{i,t}$, and “pulls a message” from itself as well. Recall that each client $i$ can choose its mixing weights ($i$th row of $P^t$) independently of the other clients. So in order to provide more flexible collaboration and closer ties for clients, we randomly choose the in-neighbors under the communication bandwidth limitation. We use uniform mixing weights for the pulled models here, meaning that clients assign uniform model weights to all neighbors. So assuming that each client can pull models with up to $n$ neighbors, each row $P^t_i$ of $P^t$ has exactly $n+1$ non-zero entries, both of which are equal to $1/(n+1)$. 
Thus, we get that
\begin{equation} \label{ap:mixing matrices}
    p_{i,j}^t= \begin{cases}1/(n+1), & j \in \mathcal{N}_{i,t}^{in}; \\ 
    0, & \text { otherwise. }\end{cases}
\end{equation}

\noindent
\textbf{Undirected random graph.} For the undirected DPFL methods (i.e. DFedAvgM and Dis-PFL), we use a time-varying and undirected random graph to represent the inter-client connectivity. Clients randomly choose their in-neighbors to pull the shared models and push a message in return. We adopt these graphs to be consistent with the experimental setup used in \citep{Rong2022DisPFL,shi2023towards,sun2022decentralized}. So the mixing matrics in the undirected graph is a symmetric doubly-stochastic (each row and each column must sum to 1), which satisfies $p_{i,j}^t = p_{j,i}^t$ in Formula (\ref{ap:mixing matrices}). Notably, the model communication bandwidth of in-neighbors in DPFL is strictly limited as the same as the busiest server in CPFL. 

\section{More details in the experiments}\label{ap:baseline}

In this section, we provide more details of our experiments including datasets, baselines, and more extensive experimental results to compare the performance of the proposed DFedPGP against other baselines on the Tiny-ImageNet dataset. All our experiments are trained and tested on a single Nvidia RTX3090 GPU under the environment of Python 3.8.5, PyTorch 1.11.1, CUDA 11.6, and CUDNN 8.0.

\subsection{More Details about Baselines} \label{exp:baselines}
\textbf{Local} is the simplest method for personalized learning. It only trains the personalized model on the local data and does not communicate with other clients. For the fair competition, we train 5 epochs locally in each round.

\noindent
\textbf{FedAvg} \cite{mcmahan2017communication} is the most commonly discussed method in FL. It selects partial clients to perform local training on each dataset and then aggregates the trained models to update the global model. Actually, the local model in FedAvg is also the comparable personalized model for each client. 

\noindent
\textbf{FedPer} \cite{arivazhagan2019federated} proposes a model decoupling approach for PFL, with a consensus representation and many local classifiers, to combat the ill effects of statistical heterogeneity. We set the linear layer as the personalized layer and the rest model as the base layer.  It follows FedAvg’s training paradigm but only passes the base layer to the server and keeps the personalized layer locally.

\noindent
\textbf{FedRep} \cite{collins2021exploiting} also proposes a personalized model decoupling framework like FedPer, but it fixes one part when updating the other. We follow the official implementation\footnote{\url{https://github.com/lgcollins/FedRep}} to train the head for 10 epochs with the body fixed, and then train the body for 5 epochs with the head fixed. 

\noindent
\textbf{FedBABU} \cite{oh2021fedbabu} is also a model decoupling method that achieves good personalization via fine-tuning from a good shared representation base layer. Different from FedPer and FedRep, FedBABU only updates the base layer with the personalized layer fixed and finally fine-tunes the whole model. Following the official implementation\footnote{\url{https://github.com/jhoon-oh/FedBABU}}, it fine-tunes 5 times in our experiments. 

\noindent
\textbf{Ditto} \cite{li2021ditto} achieves personalization via a trade-off between the global model and local objectives. It totally trains two models on the local datasets, one for the global model (similarly aggregated as in FedAvg) with its local empirical risk, and one for the personal model (kept locally) with both empirical risk and the proximal term towards the global model. We set the regularization parameters $\lambda$ as 0.75. 

\noindent
\textbf{DFedAvgM} \cite{sun2022decentralized} is the decentralized FedAvg with momentum, in which clients only connect with their neighbors by an undirected graph. For each client, it first initials the local model with the received models then updates it on the local datasets with a local stochastic gradient. 

\noindent
\textbf{OSGP} \cite{assran2019stochastic} is the directed version of DFedAvg, which allows clients to send the local models to their out-neighbors by a directed graph. It is regarded as a representative of a personalized baseline over directed communication.

\noindent
\textbf{Dis-PFL} \cite{Rong2022DisPFL} employs personalized sparse masks to customize sparse local models in the PFL setting. Each client first initials the local model with the personalized sparse masks and updates it with empirical risk. Then filter out the parameter weights that have little influence on the gradient through cosine annealing pruning to obtain a new mask. Following the official implementation\footnote{\url{https://github.com/rong-dai/DisPFL}}, the sparsity of the local model is set to 0.5 for all clients.

\subsection{Datasets and Data Partition}

CIFAR-10/100 and Tiny-ImageNet are three basic datasets in the computer version study. As shown in Table \ref{ta:all_data}, they are all colorful images with different classes and different resolutions. We use two non-IID partition methods to split the training data in our implementation. One is based on Dirichlet distribution on the label ratios to ensure data heterogeneity among clients. The Dirichlet distribution defines the local dataset to obey a Dirichlet distribution (see in Figure \ref{fig:dir}), where a smaller $\alpha$ means higher heterogeneity. Another assigns each client a limited number of categories, called Pathological distribution. Pathological distribution defines the local dataset to obey a uniform distribution of active categories $c$ (see in Figure \ref{fig:pat}), where fewer categories mean higher heterogeneity. The distribution of the test datasets is the same as in training datasets. We run 500 communication rounds for CIFAR-10, CIFAR-100, and 300 rounds for Tiny-ImageNet.
% \subsection{Communication Topology}

\begin{table}[ht]
\centering
\small
\caption{ \small  The details on the CIFAR-10 and CIFAR-100 datasets.}
\label{ta:all_data}
\begin{tabular}{ccccc} 
\toprule
Dataset       & Training Data & Test Data & Class & Size     \\ 
\midrule
CIFAR-10      & 50,000        & 10,000    & 10    & 3×32×32  \\
CIFAR-100     & 50,000        & 10,000    & 100   & 3×32×32  \\
Tiny-ImageNet & 100,000       & 10,000    & 200   & 3×64×64        \\
\bottomrule
\end{tabular}
\end{table}

\begin{figure}[h]
%\vspace{-0.3cm}
	\centering
	\begin{minipage}{0.45\linewidth}
		\centering
		\includegraphics[width=1\linewidth]{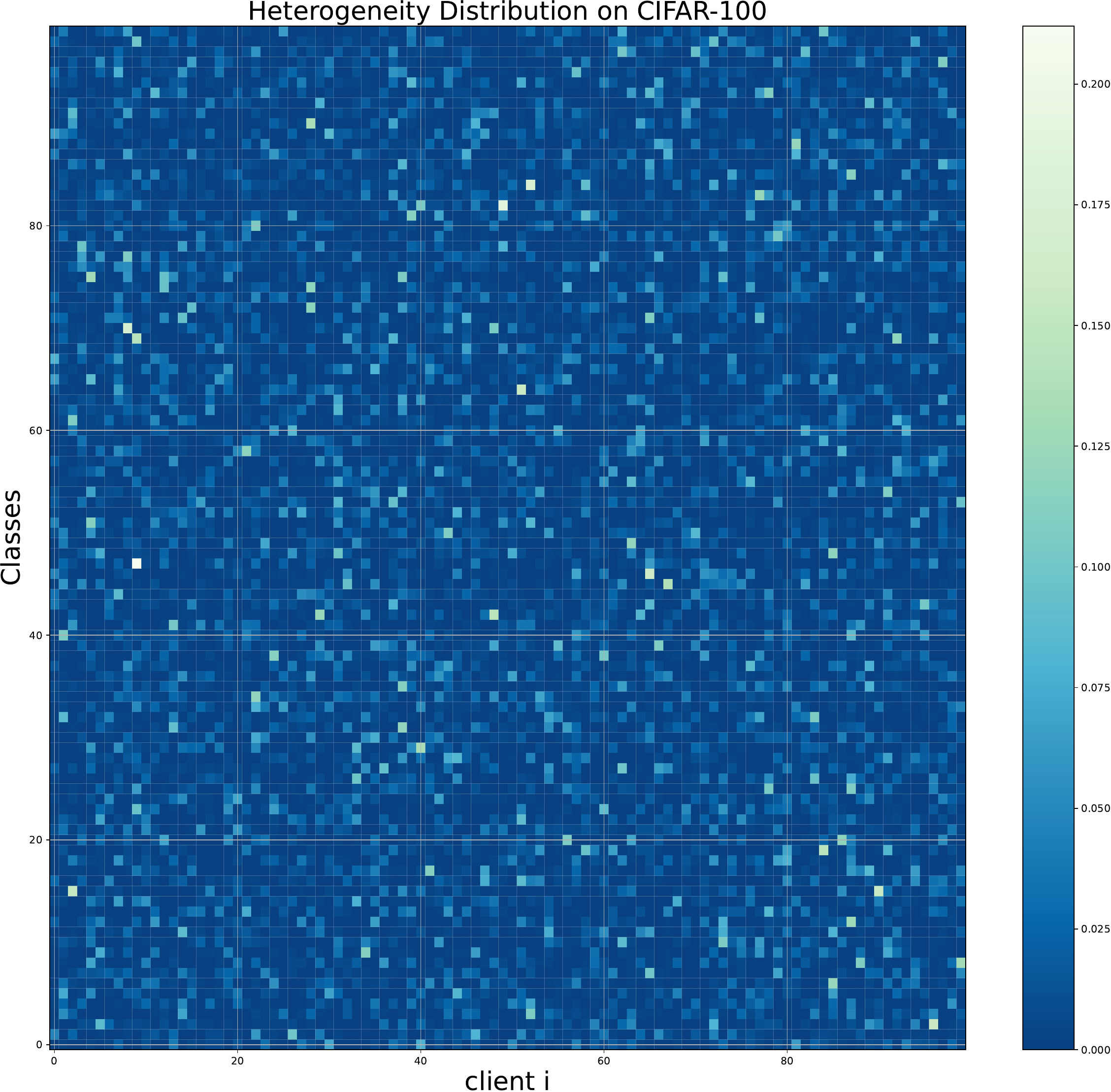}
		\subcaption{Dirichlet $\alpha=0.3$ on CIFAR-100.}
		\label{fig:dir}%文中引用该图片代号
	\end{minipage}
	%\qquad
	\begin{minipage}{0.45\linewidth}
		\centering
		\includegraphics[width=1\linewidth]{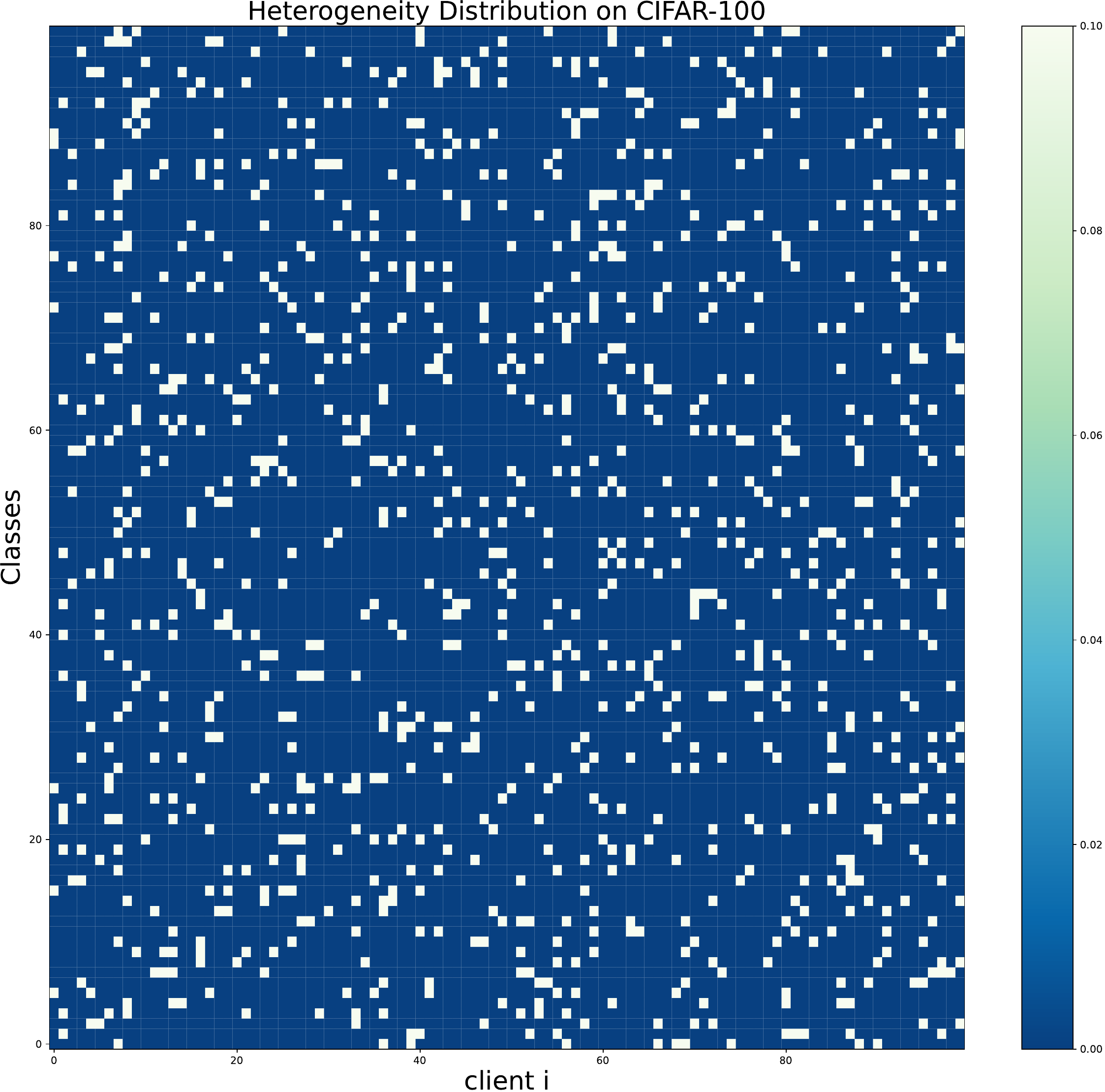}
		\subcaption{Pathological $c = 10$ on CIFAR-100.}
		\label{fig:pat}
	\end{minipage}

 \caption{ Heat-map of the Dirichlet split and Pathological split.}
\end{figure}

\subsection{More Experiments Results on Tiny ImageNet}\label{exper:tiny}
\textbf{Comparison with the baselines.} In Table \ref{ap:ta_baselines} and Figure \ref{fig:img_tinybaseline}, we compare DFedPGP with other baselines on the Tiny-ImageNet with different data distributions. The comparison shows that the proposed method has a competitive performance, especially under higher heterogeneity, e.g. Pathological-10. Specifically in the Pathological-10 setting, DFedPGP achieves 49.16\%, at least 1.81\% and 7.08\%  improvement from the CFL methods and DFL methods. However, in the Dirichlet-0.3 setting, almost all the partial model personalized methods (i.e. FedPer, FedRep, DFedPGP except FedBABU) face a severe performance degradation compared with the full model personalized methods (i.e. FedAvg, DFedAvgM, OSGP). This may account for the low classification ability in partial model personalized methods without aggregation with neighbors in the multiple-image classification tasks, especially in the long-tail data distribution scenario (i.e. Dirichelet-0.3). The original intention of our design is to build a great personalized model through partial model personalization training and directed collaboration with neighbors. So when the heterogeneity increases, our algorithms have a significant improvement.

\begin{figure} [h]
\centering
\includegraphics[width=1\textwidth]{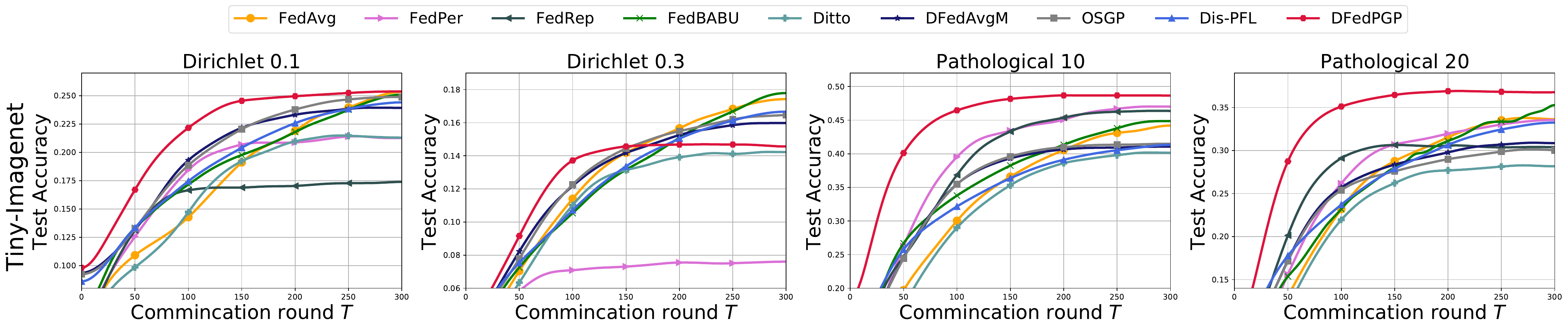}
\centering
%\vspace{-0cm}
\caption{\small Test accuracy on Tiny-ImageNet with heterogenous data partitions. }
\label{fig:img_tinybaseline}
%\vspace{-0.35cm}
\end{figure}

\begin{table}[h]
\centering \small
\caption{ \small  Test accuracy (\%) on Tiny-ImageNet in both Dirichlet and Pathological distribution settings on Tiny-ImageNet.}
\label{ap:ta_baselines}
\begin{tabular}{lcccc} 
\toprule
\multirow{3}{*}{Algorithm} & \multicolumn{4}{c}{Tiny-ImageNet}                                                                            \\ 
\cmidrule{2-5}
                           & \multicolumn{2}{c}{Dirichlet}                       & \multicolumn{2}{c}{Pathological}                      \\ 
\cmidrule{2-5}
                           & $\alpha$ = 0.1          & $\alpha$ = 0.3            & c = 10                    & c = 20                     \\ 
\midrule
Local                      & $12.13_{\pm.13}$          & $5.42_{\pm.21} $        & $28.49_{\pm.16} $           & $16.72_{\pm.34} $            \\
FedAvg                     & $25.55_{\pm.02}$          & $17.58_{\pm.25}$            & $44.56_{\pm.39}$            & $ 34.10_{\pm.59}$             \\
FedPer                     & $21.64_{\pm.72}$          & $7.71_{\pm.08} $            & $47.35_{\pm.03}$            & $ 33.68_{\pm.33}$             \\
FedRep                     & $17.54_{\pm.79}$          & $5.78_{\pm.05}$            & $46.76_{\pm.73} $           & $ 31.15_{\pm.54}$             \\
FedBABU                    & $25.59_{\pm.08}$          & $\textbf{18.18}_{\pm.06}$            & $46.53_{\pm.20} $           & $ 37.01_{\pm.31} $           \\
Ditto                      & $21.71_{\pm.66}$          & $14.47_{\pm.14}$            & $40.65_{\pm.15} $           & $ 28.74_{\pm.38} $           \\ 
\midrule
DFedAvgM                   & $24.42_{\pm.74}$          & $16.51_{\pm.68} $           & $41.94_{\pm.37} $           & $31.50_{\pm.46}$         \\
OSGP                   & $25.29_{\pm.26}$          & $17.07_{\pm.17} $           & $42.08_{\pm.43} $           & $30.58_{\pm.51}$         \\
Dis-PFL                     & $24.71_{\pm.18}$          & $16.94_{\pm.36}$            & $41.93_{\pm.12}$            & $33.57_{\pm.62}$        \\
\midrule
DFedPGP                    & $\textbf{25.71}_{\pm.20}$          & $14.94_{\pm.44}$          & $\textbf{49.16}_{\pm.19}$         & $\textbf{37.25}_{\pm.27}  $            \\
\bottomrule
\end{tabular}
\end{table}

\noindent
%\vspace{1cm}

\begin{table}[h]
\centering \small
\caption{ \small The required communication rounds when achieving the target accuracy (\%) on Tiny-ImageNet.}
\label{ap:ta_convergency speed}
\begin{tabular}{lcc|cc|cc|cc} 
\toprule
\multirow{3}{*}{Algorithm} & \multicolumn{8}{c}{Tiny-ImageNet}                                                                                     \\ 
\cmidrule{2-9}
                           & \multicolumn{2}{c|}{Dirichlet-0.1} & \multicolumn{2}{c|}{Dirichlet-0.3} & \multicolumn{2}{c|}{Pathological-10} & \multicolumn{2}{c}{Pathological-20}  \\ 
\cmidrule{2-9}
                           & acc@20      & speedup       & acc@14 & speedup             & acc@40      & speedup      & acc@30      & speedup       \\ 
\midrule
FedAvg                     & 160         & 1.11 ×         & 144    & 1.47 ×              & 192         & 1.36 ×        & 172         & 1.50 ×         \\
FedPer                     & 123         & 1.45 ×         & -      & -                   & 103         & 2.53 ×        & 134         & 1.93 ×         \\
FedRep                     & -           & -             & -      & -                   & 116         & 2.25 ×        & 117         & 2.21 ×         \\
FedBABU                    & 156         & 1.14 ×         & 174    & 1.22 ×               & 178         & 1.47 ×        & 181         & 1.43 ×         \\
Ditto                      & 178         & 1.00 ×            & 212    & 1.00 ×                  & 261         & 1.00 ×           & -           & -             \\ 
\midrule
DFedAvgM                   & 110        & 1.62 ×         & 141    & 1.50 ×            & 173         & 1.51 ×        & 210       & 1.23 ×            \\
OSGP                   & 115         & 1.55 ×         & 136    & 1.56 ×            & 160         & 1.63 ×        & 258       & 1.00 ×            \\
Dis-PFL                    & 143         & 1.24 ×         & 166    & 1.28 ×            & 227         & 1.15 ×        & 188       & 1.37 ×               \\

\midrule
DFedPGP                    &\textbf{ 74  }        & \textbf{2.41 × }        & \textbf{108 }   & \textbf{1.96 × }              & \textbf{54} & \textbf{4.83 ×}      & \textbf{53} & \textbf{4.87 × }        \\
\bottomrule
\end{tabular}
\end{table}

\noindent
\textbf{Convergence speed.} We show the convergence speed of DFedPGP in Table \ref{ap:ta_convergency speed} and Figure \ref{fig:img_tinybaseline} by reporting the number of rounds required to achieve the target personalized accuracy (acc@) on Tiny-ImageNet. We set the algorithm that takes the most rounds to reach the target accuracy as “1.00×”, and find that the proposed DFedPGP achieves the fastest convergence speed on average (3.51× on average) among the SOTA PFL algorithms. Direct communication guarantees flexible choice of neighbors and closer ties between clients, which speeds up personalized convergence and achieves higher personalized performance for each client. Also, the partial model personalization and alternate updating mode will both bring a comparable gain to the convergence speed from the difference between DFedPGP and OSGP. Thus, our methods can efficiently train the personalized model under different data heterogeneity.

\subsection{More Details about hyperparameters selection}\label{exper:hyperparameters}
Here we detail the hyperparameter selection in our experiments. We fix the total communication rounds T, mini-batch size and weight decay for all the benchmarks and our proposed DFedPGP. The other selections are stated as follows.

% \usepackage{tabularray}
% \usepackage{booktabs}
%\vspace{-0.3cm}
\begin{table}[h]
%\vspace{-0.3cm}
\caption{ \small General hyperparameters introductions.}
\centering
\small
\begin{tabular}{c|cc} 
\toprule
Hyperparameter               & CIFAR-10/100, Tiny-ImageNet & Best Selection  \\ 
\midrule
Communication Round & 500                         & -               \\
Batch Size           & 128                         & -               \\
Weight Decay         & 5e-4                        & -               \\
Momentum             & 0.9                         & -               \\
Learning Rate Decay  & 0.9                         & -               \\ 
\midrule
Local Interval      & {[}1, 3, 5, 8]              & 5               \\
Local Learning Rate  & {[}0.01, 0.1, 0.5, 1]       & 0.1             \\
\bottomrule
\end{tabular}
\end{table}

\section{Proof of Theoretical Analysis}\label{ap:proof}

\subsection{Preliminary Lemmas}

\begin{lemma}[Local update for personalized model $v_i$ in DFedPGP, Lemma 23 \citep{pillutla2022federated}]
\label{le:v} 
    Consider $F$ which is $L$-smoothness and fix $v^0 \in \mathbb{R}^d$. 
	Define the sequence $(v^{k})$ of iterates produced by stochastic gradient descent with a fixed learning rate $\eta_v \leq 1/(2K_vL_v)$
	starting from $v^{0}$, we have the bound   
 
	\[
		\E \| v^{K_v-1} - v^{0} \|^2 \le 16 \eta_v^2 K_v^2 \E \|\nabla  F(v^{0})\|^2 + 8 \eta_v^2 K_v^2 \sigma_v^2 \,.
        \]
\end{lemma}

\begin{proof}

    \begin{equation}
        \begin{split}
         \E \| v_i^{t,k+1} - v_i^{t,0} \|^2 
        & =  \E \Big\| v_i^{t,k} - \eta_v \nabla_v F_i(z_i^{t}, v_i^{t,k};\xi_i) - v_i^{t,0} \Big\|^2 \\
        &   \overset{a)}\leq \Big(1+\frac{1}{K_v -1}\Big) \E \| v_i^{t,k} - v_i^{t,0}\|^2 + K_v \eta_v^2 \E \Big\| \nabla_v F_i(z_i^{t}, v_i^{t,k};\xi_i) 
        - \nabla_v F_i(z_i^{t}, v_i^{t})+\nabla_v F_i(z_i^{t}, v_i^{t}) \Big\|  \\
        &  \leq \Big(1+\frac{1}{K_v -1}\Big) \E \Big\| v_i^{t,k} - v_i^{t,0} \Big\|^2  +  K_v \eta_v^2 \Big( \sigma_v^2 
        + \E \Big\|\nabla_v F_i(z_i^{t}, v_i^{t}) - \nabla_v F_i(z_i^{t}, v_i^{t,0}) + \nabla_v F_i(z_i^{t}, v_i^{t,0}) \Big\|^2 \Big)\\
        &  \overset{b)}\leq \Big(1+\frac{1}{K_v -1}\Big) \E \| v_i^{t,k} - v_i^{t,0} \|^2 
        +  K_v \eta_v^2 \sigma_v^2  
        + 2 K_v \eta_v^2 L_v^2 \| v_i^{t,k} - v_i^{t,0} \|^2
        + 2 K_v \eta_v^2 \big\|\nabla_v F_i(z_i^t,v_i^{t,0} )\big\|^2 \\
        & \leq \Big(1+\frac{1}{K_v -1} +2 K_v \eta_v^2 L_v^2\Big) \E \| v_i^{t,k} - v_i^{t,0} \|^2 +  K_v \eta_v^2 \sigma_v^2  + 2 K_v \eta_v^2 \|\nabla_v F_i(z_i^t,v_i^{t,0} )\|^2 \\
        &  \overset{c)} \leq \Big(1+\frac{2}{K_v -1}\Big) \E \| v_i^{t,k} - v_i^{t,0}\|^2 +  K_v \eta_v^2 \sigma_v^2  + 2 K_v \eta_v^2 \|\nabla_v F_i(z_i^t,v_i^{t,0} )\|^2 .
        \end{split}
    \end{equation}
where we used a) the inequality $2\alpha\beta \leq \alpha/K + K\beta$ for reals $\alpha, \beta, K$; b) $L$-smoothness of $F$, and c) the condition on the learning rate $\eta_v \leq 1/(2K_vL_v)$. Let $A = K_v \eta_v^2 \sigma_v^2  + 2 K_v \eta_v^2 \|\nabla_v F_i(z_i^t,v_i^{t,0} )\|^2  $. Unrolling the inequality and summing up the series gives for all $k \leq K_v-1 $:

\begin{equation}
    \begin{split}
        \E \| v_i^{t,k+1} - v_i^{t,0} \|^2 
        &  \leq \Big(1+\frac{2}{K_v -1} \Big) \E \| v_i^{t,k} - v_i^{t,0} \|^2 + A \\
        & \leq A \sum_{k=0}^{K_v-1} \Big( 1+ \frac{2}{k-1} \Big)^k 
        \leq \frac{A}{2}(K_v-1)\sum_{k=0}^{K_v-1} \Big( 1+ \frac{2}{K_v-1} \Big)^k \\
        & \leq \frac{A}{2}(K_v-1) \Big( 1+\frac{2}{K_v-1} \Big)^{K_v-1} .
    \end{split}
\end{equation}
Using the bound $(1+2/K_v-1)^{K_v-1} \leq e^2 <8 $ for $K_v > 1 $, we have: 
\begin{equation}
    \begin{split}
         \E \| v_i^{K_v-1} - v_i^{0} \|^2 
        &  \leq 4A(K_v-1) 
        \leq 16 \eta_v^2 K_v^2 \E \big\|\nabla  F(v^{0})\big\|^2 + 8 \eta_v^2 K_v^2 \sigma_v^2.
    \end{split}
\end{equation}
\end{proof}

\begin{lemma}[Local update for shared model $u_i$ in DFedPGP]
\label{lemma:local-update}
For all clients $i \in \{1,2,...,m\}$ and local iteration steps $k \in \{0,1,..., K_u-1\}$, assume that assumptions 2-4 hold and define $\nabla_u F_i (z_i^{t,k}, v_i^{t+1};\xi_i) = \nabla_u F_i (u_i^{t,k}/\mu_i^{t}, v_i^{t+1};\xi_i) $, we can get   

\begin{equation}
    \frac{1}{m}\sum_{i=1}^m \E \bigl\| u_i^{t,k} - u_i^t\bigr \|^2 
    \leq 32K_u\eta_u^2\sigma_u^2  + 32K_u\eta_u^2\sigma_g^2  
         +\frac{32 K_u\eta_u^2}{m}\sum_{i=1}^m \E\|\nabla f(z_{i}^t,V^{t+1})\|^2.
\end{equation}

\begin{proof}
 
\begin{equation}
    \begin{split}
        \E \bigl\| u_i^{t,k+1} - u_i^t\bigr \|^2 
        & =  \E \Big\| u_i^{t,k} - \eta_u \nabla_u F_i(z_i^{t,k}, v_i^{t+1};\xi_i) - u_i^t \Big\|^2 \\
        & \leq (1+\frac{1}{2K_u -1}) \E \| u_i^{t,k} - u_i^{t}\|^2 + 2K_u \eta_u^2 \E \Big\| \nabla_u F_i(z_i^{t,k}, v_i^{t+1};\xi_i) \Big\|^2 \\
        & \leq (1+\frac{1}{2K_u -1}) \E \| u_i^{t,k} - u_i^{t}\|^2 + 2K_u \eta_u^2 \E \Big\| \nabla_u F_i(z_i^{t,k}, v_i^{t+1};\xi_i) 
        - \nabla_u F_i(z_i^{t,k}, v_i^{t+1}) \\
        & + \nabla_u F_i(z_i^{t,k}, v_i^{t+1})  - \nabla_u F(z_i^{t,k}, V^{t+1}) +\nabla_u F(z_i^{t,k}, V^{t+1}) - \nabla_u F(z_i^{t}, V^{t+1})  
        + \nabla_u F(z_i^{t}, V^{t+1})
         \Big\|^2 \\
        & \leq \Big(1+\frac{1}{2K_u -1}\Big) \E \| u_i^{t,k} - u_i^{t} \|^2 + 8 K_u \eta_u^2 \Big( \sigma_u^2 + \sigma_g^2 + L_u^2 \E \|z_i^{t,k}-z_i^t\|^2 
        + \E \| \nabla_u F( z_i^t,V^{t+1} )\|^2 \Big).
    \end{split}
\end{equation}
where we use Assumption \ref{assmp:stoc-grad-var}, \ref{assmp:grad-diversity.} and $L$-smoothness in the last inequation.

In addition, according to line 11 of Algorithm \ref{DFedPGP}, we can obtain $ \E\|z_{i}^{t,k} - z_i^t\|^2 =\frac{1}{\|\mu_i^t\|^2} \E\|u_{i}^{t,k} - u_i^t\|^2$. According to Property 2.1 by \citep{taheri2020quantized}, there exists $\delta > 0$  that satisfies $ \|\mu_i^t\| > \delta$. Therefore, we can get $\E\|z_{i}^{t,k} - z_i^t\|^2 \leq \frac{1}{\delta^2} \E\|u_{i}^{t,k} - u_i^t\|^2$. Assume the learning rate $0<\eta_u<\frac{\delta}{8L_uK_u}$, then we have

 \begin{equation}
 \begin{split}
    \E\|u_{i}^{t,k+1}-u_i^t\|^2 
    &\leq \Big(1+\frac{1}{2K_u-1}
    +\frac{8K_u L_u^2\eta_u^2}{\delta^2}\big)\E\|u_{i}^{t,k} - u_i^t\|^2 
    + 8K_u\eta_u^2\sigma_u^2 + 8K_u\eta_u^2\sigma_g^2  
    + 8 K_u\eta_u^2\E\|\nabla f(z_{i}^t,V^{t+1})\|^2\\
    &  \leq \Big(1+\frac{1}{K_u-1}\Big)\E\|u_{i}^{t,k} - u_i^t\|^2 
    + 8K_u\eta_u^2\sigma_u^2 + 8K_u\eta_u^2\sigma_g^2  
    + 8 K_u\eta_u^2\E\|\nabla f(z_{i}^t,V^{t+1})\|^2\\
    &\leq \sum_{k=0}^{K_u-1}\Big(1+\frac{1}{K_u-1}\Big)^k \big( 8K_u\eta_u^2\sigma_u^2 
    + 8K_u\eta_u^2\sigma_g^2  
    + 8 K_u\eta_u^2\E\|\nabla f(z_{i}^t,V^{t+1})\|^2\big)\\
    &\leq (K_u-1)\Big((1+\frac{1}{K_u-1})^{K_u}-1\Big) \times \Big( 8K_u\eta_u^2\sigma_u^2 
    + 8K_u\eta_u^2\sigma_g^2  
    + 8 K_u\eta_u^2\E\|\nabla f(z_{i}^t,V^{t+1})\|^2\Big)\\
    \vspace{1cm}
    &\leq 32K_u\eta_u^2\sigma_u^2 
    + 32K_u\eta_u^2\sigma_g^2  
    + 32 K_u\eta_u^2\E\|\nabla f(z_{i}^t,V^{t+1})\|^2.
\end{split}
\end{equation}
where we use the inequality $(1+{\frac{1}{K_u-1}})^{K_u}\leq5$ holds for any $K_u> 1$ in the last equation. Summing up from $i=1$ to $m$, then we complete the proof.
    
\end{proof}
\end{lemma}

\begin{lemma}[Mixing connectivity \citep{assran2020asynchronous}]\label{th:lemma3}  Suppose the time-varying communication topology is strongly connected. It holds for $\forall i \in \{1,\cdots ,m\}$ and $t \geq$ 0 that
 
\begin{equation}
 \frac{1}{m} \sum_{i=1}^m \E \|{ {\bar{u}^t} - z_i^{t} } \|^2 
 \leq \frac{8K_u^2 \eta_u^2 C^2}{(1-q)^2K_u-8K_u^2 \eta_u^2 L_u^2 C^2}
 \Big(\sigma_u^2  + \sigma_g^2 + \E[\Delta_{\Bar{u}}^t] \Big).
\end{equation}

\begin{proof}

Suppose that Assumption \ref{assmp:mixing connectivity} holds. Let $\lambda = 1 - n D^{-(K_{u} + 1) \Delta B}$ and let $q = \lambda^{1/((K_{u}+ 1) \Delta B + 1)}$. Then there exists a constant $C$, it satisfies

\begin{equation}
    C < \frac{2 \sqrt{d} D^{(K_{u} + 1) \Delta B}}{\lambda^{\frac{(K_{u} + 1) \Delta B + 2}{(K_{u} + 1) \Delta B + 1}}}.
\end{equation}
where $ d $ is the dimension of $ \bar{u}^{t}$, $z_i^{t}$, and $ u_i^{0} $, such that, for all $ i=1,2,\dots,m$ (non-virtual nodes) and $ t\geq0 $, 

\begin{align}\label{eq:mixing_connect}
 \|{ {\bar{u}^t} - z_i^{t} } \| \leq C q^t \|{ u_i^{0}}\|  
+ \eta_u C \sum^t_{j=1} q^{t-j} 
\|{ \sum_{k=0} ^{K_u-1}\nabla_u F_{i}(z_i^{t,k},v^{j+1}_i;\xi_i)\|}. 
\end{align}

To unfold the stochastic gradient item, we get%, we can obtain the following expressing:

%\begin{equation}
%    \begin{split}
%        \Big\| \nabla_u F_{i}(z_i^{t,k},v^{j+1}_i;\xi_i) \Big\|^2 
%        & \leq  \Big\| \nabla_u F_{i}(z_i^{t,k},v^{j+1}_i;\xi_i) 
%        - \nabla_u F_{i}(z_i^{t,k},v^{j+1}_i) 
%        +\nabla_u F_{i}(z_i^{t,k},v^{j+1}_i) \Big\|^2 \\
%        \vspace{0.6cm}
%        &\leq B^2+\sigma_u^2
%    \end{split}
%\end{equation}

\begin{equation} \label{eq:zv_xi}
    \begin{split}
        \E \Big\| \nabla_u F_{i}(z_i^{t,k},v^{j+1}_i;\xi_i) \Big\|^2 
        & \leq \Big\| \nabla_u F_i(z_i^{t,k}, v_i^{t+1};\xi_i) 
        - \nabla_u F_i(z_i^{t,k}, v_i^{t+1}) 
        + \nabla_u F_i(z_i^{t,k}, v_i^{t+1})  
        - \nabla_u F(z_i^{t,k}, V^{t+1}) \\
        & +\nabla_u F(z_i^{t,k}, V^{t+1}) 
        - \nabla_u F(z_i^{t}, V^{t+1})  
        + \nabla_u F(z_i^{t}, V^{t+1}) \Big\|^2  \\
        %- \nabla_u F(\bar{u}^t, V^{t+1})  
        %+ \nabla_u F(\bar{u}^t, V^{t+1})  \\
        & \leq 4\sigma_u^2 + 4\sigma_g^2 
        + 4L_u^2 \E \|z_i^{t,k}-z_i^t\|^2  
        + 4\E \| \nabla_u F( z_i^t,V^{t+1} )\|^2  \\
        & \leq 4 \sigma_u^2 + 4\sigma_g^2 
        + \frac{4L_u^2}{\delta^2} \E \|u_i^{t,k}-u_i^t\|^2  
        + 4\E \| \nabla_u F( z_i^t,V^{t+1} )\|^2  \\ 
        & \overset{a)}\leq 4 \sigma_u^2 + 4\sigma_g^2 
        + \frac{128K_u L_u^2\eta_u^2}{\delta^2} \Big( \sigma_u^2  + \sigma_g^2  
        + \E\|\nabla f(z_{i}^t,V^{t+1}))\|^2 \Big)
        + 4\E \| \nabla_u F( z_i^t,V^{t+1} )\|^2  \\  %这里加解释是从lemma2来的
        & \leq 4\Big(1+\frac{32K_u L_u^2\eta_u^2}{\delta^2}\Big) \Big(\sigma_u^2  + \sigma_g^2 + \E \| \nabla_u F( z_i^t,V^{t+1} )\|^2 \Big).
    \end{split}
\end{equation}
where a) uses Lemma \ref{lemma:local-update}. Focusing on the last term we have:

\begin{equation} \label{eq:zV}
    \begin{split}
      \E \| \nabla_u F( z_i^t,V^{t+1} )\|^2 
      & \leq \E \| \nabla_u F( z_i^t,V^{t+1} ) 
      - \nabla_u F(\bar{u}^t, V^{t+1}) 
      + \nabla_u F(\bar{u}^t, V^{t+1}) \| ^2   \\
      & \leq  L_u^2 \E \|{ {\bar{u}^t} - z_i^{t} } \|^2 
      + \E[\Delta_{\Bar{u}}^t].
    \end{split}
\end{equation}

Substituting Formula (\ref{eq:zV}) and (\ref{eq:zv_xi}) into (\ref{eq:mixing_connect}), then squaring both sides and taking expectations, we have
\begin{equation}
    \begin{split}
       \E \|{ {\bar{u}^t} - z_i^{t} } \|^2 
        & \leq \big( C q^t \|{ u_i^{0}}\|  
         + \eta_u C \sum^t_{j=1} q^{t-j} 
        \E \|{ \sum_{k=0} ^{K_u-1}\nabla_u F_{i}(z_i^{t,k},v^{j+1}_i;\xi_i)\| }  \big)^2 \\
        & \overset{a)}  \leq 2C^2 q^{2t} \|{ u_i^{0}}\|^2  
        + 2\eta_u^2 C^2 \big( \sum^t_{j=1} q^{t-j} 
        \E \|{ \sum_{k=0}^{K_u-1} \nabla_u F_{i}(z_i^{t,k},v^{j+1}_i;\xi_i) }\| \big) ^2 \\
        & \leq 2C^2 q^{2t} \|{ u_i^{0}}\|^2  
        + \frac{2 K_u^2 \eta_u^2 C^2}{(1-q)^2}  
        \E \| \nabla_u F_{i}(z_i^{t,k},v^{j+1}_i;\xi_i) \|^2\\
        & \leq 2C^2 q^{2t} \|{ u_i^{0}}\|^2  
        + \frac{8 K_u^2 \eta_u^2 C^2}{(1-q)^2}  
        \Big(1+\frac{32K_u L_u^2\eta_u^2}{\delta^2}\Big) \Big(\sigma_u^2  + \sigma_g^2 + \E \| \nabla_u F( z_i^t,V^{t+1} )\|^2 \Big) \\
        & \leq 2C^2 q^{2t} \|{ u_i^{0}}\|^2  
        + \frac{8 K_u^2 \eta_u^2 C^2}{(1-q)^2}  
        \Big(1+\frac{32K_u L_u^2\eta_u^2}{\delta^2}\Big) \Big(\sigma_u^2  + \sigma_g^2 
        + L_u^2 \E \|{ {\bar{u}^t} - z_i^{t} } \|^2 
        + \E[\Delta_{\Bar{u}}^t] \Big). 
    \end{split}
\end{equation}
where a) uses $<x,y> \leq \frac{1}{2}\|x\|^2 + \frac{1}{2}\|y\|^2$. 

Move $\E \|{ {\bar{u}^t} - z_i^{t} } \|^2$ to the left side of the inequality and assume $\| u_i^0 \| = 0$ and $0 < \eta_u < \frac{\delta}{4\sqrt{2}K_u L_u}$, then we have

\begin{equation}
    \begin{split}
       \E \|{ {\bar{u}^t} - z_i^{t} } \|^2 
        & \leq  \frac{8K_u^2 \eta_u^2  C^2 (K_u+1)}
        {(1-q)^2K_u-8K_u^2 \eta_u^2 L_u^2 C^2 (K_u+1)}
        \Big(\sigma_u^2  + \sigma_g^2 
        + \E[\Delta_{\Bar{u}}^t] \Big) .
    \end{split}
\end{equation}

Summing up from $i=1$ to $m$, then we complete the proof.
\end{proof}
\end{lemma}

\subsection{Proof of Convergence Analysis}
\textbf{Proof Outline and the Challenge of Dependent Random Variables.}
    We start with 
    \begin{align} 
        \begin{aligned}
        F\left(\Bar{u}^{t+1}, V^{t+1}\right)
        - F\left(\Bar{u}^{t}, V^{t}\right)
        =&\, F\left(\Bar{u}^{t}, V^{t+1}\right)
        - F\left(\Bar{u}^{t}, V^{t}\right) + F\left(\Bar{u}^{t+1}, V^{t+1}\right)
        - F\left(\Bar{u}^{t}, V^{t+1}\right) .
        \end{aligned}
    \end{align}
    
    The first line corresponds to the effect of the $v$-step and the second line to the $u$-step. The former is 
    \begin{equation}
    \begin{aligned}
        F\left(\Bar{u}^{t}, V^{t+1}\right)- F\left(\Bar{u}^{t}, V^{t}\right) & = \frac{1}{m}\sum_{i=1}^m \E \Big[F_i(\Bar{u}^t, v_i^{t+1} ) - F_i(\Bar{u}^t, v_i^{t} )\Big]\\
        & \le  \frac{1}{m}\sum_{i=1}^m \E \Big[\Big <\nabla_v F_i\left(\Bar{u}
        ^{t}, v^{t}_i\right), v^{t+1}_i - v^{t}_i \Big>
        + \frac{L_v}{2}\|v^{t+1}_i - v^{t}_i \|^2 \Big] .
    \end{aligned}
    \end{equation}
    
    It is easy to handle with standard techniques that rely on the smoothness of $F\left(u^{t}, \cdot\right)$. 
    The latter is more challenging. 
    In particular, the smoothness bound for the $u$-step gives us                                          
    \begin{align}
        F&\left(\Bar{u}^{t+1}, V^{t+1}\right)
        - F\left(\Bar{u}^{t}, V^{t+1}\right)
        \le 
        \Big <\nabla_u F\left(\Bar{u}
        ^{t}, V^{t+1}\right), \Bar{u}^{t+1} - \Bar{u}^{t} \Big>
        + \frac{L_u}{2}\|\Bar{u}^{t+1} - \Bar{u}^{t}\|^2 \,.
    \end{align}
\subsubsection{Proof of Convergence Analysis for DFedPGP}
\textbf{Analysis of the $u$-Step.}
    \begin{equation}
        \begin{aligned}
        & \E \Big [F\left(\Bar{u}^{t+1}, V^{t+1}\right)
        - F\left(\Bar{u}^{t}, V^{t+1}\right) \Big]
         \le 
        \Big <\nabla_u F\left(\Bar{u}
        ^{t}, V^{t+1}\right), \Bar{u}^{t+1} - \Bar{u}^{t} \Big>
        + \frac{L_u}{2}\E\|\Bar{u}^{t+1} - \Bar{u}^{t}\|^2\\
        & \leq \frac{-\eta_u}{m}\sum_{i=1}^m\E\Big <\nabla_u F\left(\Bar{u}
        ^{t}, V^{t+1}\right), \sum_{k=0}^{K_u-1}\nabla_u F\left(z_i^{t,k}, v_i^{t+1}; \xi_i\right)\Big> 
        + \frac{L_u}{2}\E\|\Bar{u}^{t+1} - \Bar{u}^{t}\|^2\\
        & \leq -\eta_uK_u \E [\Delta_{\Bar{u}}^t] + \frac{\eta_u}{m}\sum_{i=1}^m\sum_{k=0}^{K_u-1} \E \Big< \nabla_u F\left(\Bar{u}^{t}, V^{t+1}\right), \nabla F\left(\Bar{u}^t, v_i^{t+1}\right) - \nabla_u F\left(z_i^{t,k}, v_i^{t+1}; \xi_i\right) \Big> + \frac{L_u}{2}\E\|\Bar{u}^{t+1} - \Bar{u}^{t}\|^2 \\
        & \overset{a)}{\leq} \frac{-\eta_uK_u }{2} \E [\Delta_{\Bar{u}}^t] + \underbrace{\frac{\eta_uL_u^2}{2m}\sum_{i=1}^m\sum_{k=0}^{K_u-1} \E \| z_i^{t,k} - \Bar{u}^t \|^2}_{\mathcal{T}_{1, u}} + \underbrace{\frac{L_u}{2}\E\|\Bar{u}^{t+1} - \Bar{u}^{t}\|^2}_{\mathcal{T}_{2, u}}.
        \end{aligned}
    \end{equation}

 Where a) uses $\E\left[ \nabla_u F(z_i^{t,k}, v_i^{t+1}; \xi_i) \right]= \nabla_u F\left(z_i^{t,k}, v_i^{t+1}\right)$ and $\left<x, y\right> \leq \frac{1}{2}\|x\|^2 + \frac{1}{2}\|y\|^2 $ for vectors $x, y$ followed by $L$-smoothness. 
 
For $\mathcal{T}_{1, u}$, we can use Lemma \ref{th:lemma3} and  set $AA =  \frac{8K_u^2 \eta_u^2 C^2 (K_u+1)}{(1-q)^2K_u-8K_u^2 L_u^2 \eta_u^2  C^2 (K_u+1)}$, then we have:
\begin{equation}\label{T1_u}
    \begin{split}
        \mathcal{T}_{1, u} \leq  \frac{K_u L_u^2 \eta_u}{2} 
        \Big(\sigma_u^2  + \sigma_g^2 
        + \E[\Delta_{\Bar{u}}^t] \Big) AA.
     \end{split}
\end{equation}

Meanwhile, for $\mathcal{T}_{2, u}$,
    \begin{equation}
    \begin{split}
     \mathcal{T}_{2, u}
     & \leq \frac{\eta_u^2L_u}{2m} \sum_{i=1}^m\sum_{k=0}^{K_u-1}\Big \|\nabla_u F\left(z_i^{t,k}, v_i^{t+1}; \xi_i\right) \Big \|^2 \\
     &\overset{a)}{\leq} \frac{\eta_u^2L_u}{2m} \sum_{i=1}^m\sum_{k=0}^{K_u-1}\Big \|\nabla_u F\left(z_i^{t,k}, v_i^{t+1}; \xi_i\right) 
      - \nabla_u F\left(z_i^{t,k}, v_i^{t+1}\right) 
     + \nabla_u F\left(z_i^{t,k}, v_i^{t+1}\right) 
     - \nabla_u F\left(z_i^{t}, v_i^{t+1}\right) \\
     &+ \nabla_u F\left(z_i^{t}, v_i^{t+1}\right)  
      + \nabla_u F\left(z_i^{t}, V^{t+1}   \right) 
    + \nabla_u F\left(z_i^{t}, V^{t+1}   \right) 
    - \nabla_u F\left(\Bar{u}^{t} , V^{t+1}\right) 
    + \nabla_u F\left(\Bar{u}^{t}, V^{t+1}\right)\Big \|^2 \\
    &  \leq \frac{5}{2}\eta_u^2K_uL_u  \Big(  \sigma_u^2 
    +  \frac{L_u^2}{m\delta^2}\sum_{i=1}^m  \E\|u_{i}^{t,k} - u_i^t\|^2 
    +  \sigma_g^2  
    +  \frac{L_u^2}{m}\sum_{i=1}^m\E \|z_i^{t}-\Bar{u}^{t}\|^2 + \E[\Delta_{\Bar{u}}^t] \Big)\\
    &  \leq \frac{5}{2}K_uL_u\eta_u^2  \Big(  \sigma_u^2  
    +  \sigma_g^2  
    + \E[\Delta_{\Bar{u}}^t] 
    +  \frac{32K_uL_u^2\eta_u^2}{\delta^2}  
    \Big(\sigma_u^2 + \sigma_g^2  + \E[\Delta_{\Bar{u}}^t]\Big)
    \Big(L_u^2AA+1\Big)
    + L_u^2 \Big( \sigma_u^2 + \sigma_g^2 + \E[\Delta_{\Bar{u}}^t] \Big)AA \Big) \\
    & \leq \frac{5}{2}K_uL_u\eta_u^2 
    \big[ 1+ \frac{32K_uL_u^2\eta_u^2}{\delta^2} (L_u^2AA+1)
    + AA \big]
    \big(\sigma_u^2 + \sigma_g^2 + \E[\Delta_{\Bar{u}}^t] \big).
    \end{split}
    \end{equation}
where we use Assumption \ref{assmp:stoc-grad-var}, \ref{assmp:grad-diversity.} and $L$-Smoothness in a). Based on the analysis above, we have:

\begin{equation}
    \begin{split}
        \E \Big [F\left(\Bar{u}^{t+1}, V^{t+1}\right)
        & - F\left(\Bar{u}^{t}, V^{t+1}\right) \Big]
         \le  \frac{K_u \eta_u}{2}  \E[\Delta_{\Bar{u}}^t]
         + \mathcal{T}_{1, u} + \mathcal{T}_{2, u}\\
         & \le \big( \frac{-\eta_uK_u}{2} 
         + \frac{K_u L_u^2\eta_u}{2}AA 
         + \frac{5K_uL_u \eta_u^2}{2}
         \big[ 1
         + \frac{32K_u L_u^2 \eta_u^2}{\delta^2} (L_u^2AA+1) + L_u^2 AA\big] \big)
         \E[\Delta_{\Bar{u}}^t]\\
         & + \big( \frac{k_u L_u^2\eta_u}{2}AA 
         + \frac{5K_uL_u \eta_u^2}{2}
         \big[ 1
         + \frac{32K_u L_u^2 \eta_u^2}{\delta^2} (L_u^2AA+1) + L_u^2 AA\big] \big) \big(\sigma_u^2 + \sigma_g^2 \big) .
    \end{split}
\end{equation}

\textbf{Analysis of the $v$-Step.}
\begin{equation}
        \begin{split}
           \E \Big [ F\left(\Bar{u}^{t}, V^{t+1}\right)
        - F\left(\Bar{u}^{t}, V^{t}\right) \Big]
        &  \le  \underbrace{\frac{1}{m}\sum_{i=1}^m \E 
        \Big <\nabla_v F_i\left(\Bar{u}
        ^{t}, v^{t}_i\right), v^{t+1}_i - v^{t}_i \Big>}_{\mathcal{T}_{1, v}}
        + \underbrace{\frac{L_v}{2m}\sum_{i=1}^m \E  \|v^{t+1}_i - v^{t}_i \|^2 }_{\mathcal{T}_{2, v}}.
        \end{split}
    \end{equation}
    
For $\mathcal{T}_{1, v}$, 
\begin{equation}\label{eq:T_1_a}
    \begin{split}
       \mathcal{T}_{1, v} & \leq \frac{1}{m}\sum_{i=1}^m \E 
        \Big <\nabla_v F_i\left(\Bar{u}
        ^{t}, v^{t}_i\right) - \nabla_v F_i\left(z_i
        ^{t}, v^{t}_i\right) + \nabla_v F_i\left(z_i
        ^{t}, v^{t}_i\right), -\eta_v \sum_{k=0}^{K_v-1} \E \nabla_v F_i(u_i^t, v^{t}_i; \xi_i) \Big> \\
        & \overset{a)}{\leq} \frac{-\eta_vK_v}{m}\sum_{i=1}^m \E \| \nabla_v F_i(u_i^t, v^{t}_i) \|^2 + \frac{1}{m}\sum_{i=1}^m \E 
        \Big <\nabla_v F_i\left(\Bar{u}
        ^{t}, v^{t}_i\right) - \nabla_v F_i\left(z_i
        ^{t}, v^{t}_i\right), v^{t+1}_i - v^{t}_i \Big> \\
        & \overset{b)}{\leq} -\eta_vK_v \E [\Delta_v^t] + \underbrace{\frac{L_{vu}^2}{2m}\sum_{i=1}^m \E \|\Bar{u}^t-z_i^{t}\|^2}_{\mathcal{T}_{3, v}} + \underbrace{\frac{1}{2m}\sum_{i=1}^m \E \|v^{t+1}_i - v^{t}_i\|^2}_{\frac{1}{L_v}\mathcal{T}_{2, v}}.
    \end{split}
\end{equation}
where a) and b) is get from the unbiased expectation property of $\nabla_v F_i(u_i^{t},v^{t}_i; \xi_i)$  and $<x, y> \leq \frac{1}{2}(\|x\|^2+\|y\|^2)$, respectively.

For $\mathcal{T}_{2, v}$, according to Lemma \ref{le:v}, we have 

\begin{equation}\label{eq:T_2_a}
    \begin{split}
       \mathcal{T}_{2, v} & \leq \frac{L_v}{2}\Big( \frac{16 \eta_v^2 K_v^2}{m}\sum_{i=1}^m \E \|\nabla_v F_i(u_i^t, v^{t}_i)\|^2 + 8 \eta_v^2 K_v^2 \sigma_v^2\Big)\\
       & \leq   8L_v\eta_v^2 K_v^2 \E [\Delta_v^t] + 4L_v\eta_v^2 K_v^2 \sigma_v^2.
    \end{split}
\end{equation}

For $\mathcal{T}_{3, v}$, according to Lemma \ref{th:lemma3}, we have 

\begin{equation}\label{eq:T_3_a}
    \frac{L_{vu}^2}{2m}\sum_{i=1}^m 
    \E \|\Bar{u}^t-z_i^{t}\|^2 
    \leq  \frac{L_{vu}^2}{2} 
    \big(\sigma_u^2 + \sigma_g^2 + \E[\Delta_{\Bar{u}}^t] \big) AA.
\end{equation}

After that, summing Formula (\ref{eq:T_1_a}), (\ref{eq:T_2_a}) and (\ref{eq:T_3_a}), we have

\begin{equation}
    \begin{split}
      \E \Big [ F\left(\Bar{u}^{t}, V^{t+1}\right)
        - F\left(\Bar{u}^{t}, V^{t}\right) \Big]  
        & \le \Big(-\eta_vK_v + 8\eta_v^2K_v^2 L_v + 8\eta_v^2K_v^2 \Big)\E [\Delta_v^t] +  4\eta_v^2K_v^2\sigma_v^2 L_v^2(1+L_v) \\
        \vspace{0.5cm}
        & + \frac{L_{vu}^2}{2} \big(\sigma_u^2 + \sigma_g^2 + \E[\Delta_{\Bar{u}}^t] \big) AA .
    \end{split}
\end{equation}
% Using Eq. (\ref{T_1_u}), we have

\textbf{Obtaining the Final Convergence Bound.} 

\begin{align} \label{eq:pfl-am:pf:2}
        \begin{aligned}
      & \E \Big [   F\left(\Bar{u}^{t+1}, V^{t+1}\right)
        - F\left(\Bar{u}^{t}, V^{t}\right)  \Big]  
        =\E \Big [  F\left(\Bar{u}^{t}, V^{t+1}\right)
        - F\left(\Bar{u}^{t}, V^{t}\right) + F\left(\Bar{u}^{t+1}, V^{t+1}\right)
        - F\left(\Bar{u}^{t}, V^{t+1}\right)  \Big]    \\
         & \leq \big( \frac{-\eta_uK_u}{2} 
         + \frac{K_u L_u^2\eta_u}{2}AA 
         + \frac{5K_uL_u \eta_u^2}{2}
         \big[ 1
         + \frac{32K_u L_u^2 \eta_u^2}{\delta^2} (L_u^2AA+1) + L_u^2 AA\big] 
         + \frac{L_{vu}^2}{2}AA \big)
         \E[\Delta_{\Bar{u}}^t] \\
        & + \Big(-\eta_vK_v + 8\eta_v^2K_v^2 L_v 
        + 8\eta_v^2K_v^2 \Big)  \E [\Delta_v^t] 
        + 4\eta_v^2K_v^2 L_v^2 \sigma_v^2 (1+L_v) \\
        & +\big( \frac{K_u L_u^2\eta_u}{2}AA 
         + \frac{5K_uL_u \eta_u^2}{2}
         \big[ 1
         + \frac{32K_u L_u^2 \eta_u^2}{\delta^2} (L_u^2AA+1) + L_u^2 AA\big] 
         + \frac{L_{vu}^2}{2}AA \big) \Big(\sigma_u^2+\sigma_g^2\Big) .
        \end{aligned}
    \end{align}
    
Summing from $t=1$ to $T$, assume the local learning rates satisfy $\eta_u=\mathcal{O}({1}/{L_uK_u\sqrt{T}}), \eta_v=\mathcal{O}({1}/{L_vK_v\sqrt{T}})$, $F^{*}$ is denoted as the minimal value of $F$, i.e., $F(\bar{u}, V)\ge F^*$ for all $\bar{u} \in \mathbb{R}^{d}$, and $V=(v_1,\ldots,v_m)\in\mathbb{R}^{d_1+\ldots+d_m}$. Assume $C^2 \ll (1-q)^2T$, then unfold $AA$, we can generate

\begin{equation}
    \begin{split}
        \frac{1}{T}\sum_{i=1}^T \bigl(\frac{1}{L_u} \E \bigl[\Delta_{\bar{u}}^t \bigr] + \frac{1}{L_v} \E [\Delta_{v}^t \bigr] \bigr) 
        & \leq \mathcal{O}\Big(\frac{F(\bar{u}^1, V^1) - F^*}{\sqrt{T}} 
        +\frac{ (1+L_v)\sigma_v^2 }{\sqrt{T}} 
        + \big(\sigma_u^2 + \sigma_g^2 \big)
        \big( \frac{ C^2}{(1-q)^2 L_u T} \\
        &   +  \frac{1}{K_u L_u \sqrt{T}} 
        + \frac{1}{K_u L_u \delta^2 T^{3/2}}
        + \frac{ C^2}{K_u L_u (1-q)^2 T^{3/2}}
        +\frac{L_{vu}^2 C^2 }{(1-q)^2 L_u^2 \sqrt{T}} \big).
    \end{split}
\end{equation}

Combining $\chi := \max\{L_{uv},\,L_{vu}\}\big/\sqrt{L_u L_v}$ in Assumption \ref{assmp:smoothness} and assume that
\begin{equation}
    \begin{split}
       \sigma_1^2 & = (1+L_v)\sigma_v^2
        + \big(\frac{1}{K_u L_u} 
        + \frac{ L_v \chi^2 C^2}{(1-q)^2 L_u} \big) 
        \big(\sigma_u^2 + \sigma_g^2\big) , \\
        \sigma_2^2 &= \frac{C^2}{(1-q)^2 L_u  }
         \big(\sigma_u^2 + \sigma_g^2\big) , \\
        \sigma_3^2 &= \big( \frac{1}{K_u L_u \delta^2}
        + \frac{ C^2}{(1-q)^2 K_u L_u} \big) \big(\sigma_u^2 + \sigma_g^2 \big). 
    \end{split}
\end{equation}

Then, we have the final convergence bound:

\begin{equation} 
    \frac{1}{T}\sum_{i=1}^T \bigl(\frac{1}{L_u} \E \bigl[\Delta_{\bar{u}}^t \bigr] 
    + \frac{1}{L_v} \E [\Delta_{v}^t \bigr] \bigr) \leq \mathcal{O}\Big(\frac{F(\bar{u}^1, V^1) - F^*}{\sqrt{T}} 
    + \frac{\sigma_1^2}{\sqrt{T}}
    + \frac{\sigma_2^2}{T}
    + \frac{\sigma_3^2}{\sqrt{T^3}} \Big).
\end{equation}